\documentclass[pdflatex]{sn-jnl}% Math and Physical Sciences Author Year Reference Style
\usepackage{graphicx}%
\usepackage{multirow}%
\usepackage{amsmath,amssymb,amsfonts}%
\usepackage{amsthm}%
\usepackage{mathrsfs}%
\usepackage[title]{appendix}%
\usepackage{xcolor}%
\usepackage{textcomp}%
\usepackage{manyfoot}%
\usepackage{booktabs}%
\usepackage{algorithm}%
\usepackage{algorithmicx}%
\usepackage{algpseudocode}%
\usepackage{listings}%
\usepackage{natbib}
\usepackage[utf8]{inputenc}
\usepackage[T1]{fontenc}
\usepackage{csquotes}
\usepackage[english]{babel}

\usepackage{dsfont}
\usepackage{amsfonts}
\usepackage{amsmath, amssymb}
\usepackage{amsthm}
\usepackage{mathrsfs}
\usepackage{hyperref}
\usepackage{graphicx}
\usepackage{float} 
\usepackage[justification=centerlast, indention=.5cm]{caption}
\usepackage[position=b]{subcaption}
\usepackage{titlesec}
\usepackage[ddmmyyyy]{datetime}
\usepackage{enumitem}
\setitemize{noitemsep}
\usepackage{placeins} 
\usepackage{indentfirst}
\usepackage{comment}
\usepackage{tikz}
\usetikzlibrary{arrows.meta, positioning}

\numberwithin{equation}{section}

\renewcommand{\bf}[1]{\textbf{#1}}
\newcommand{\norm}[1]{\left\Vert#1\right\Vert} %norme

\renewcommand{\epsilon}{\varepsilon}

\newcommand{\Xcal}{\mathcal{X}}

\newcommand{\R}{\mathbb{R}}

\newcommand{\db}{d_B}

\newcommand{\ie}{\emph{i.e.}}
\newcommand{\eg}{\emph{e.g.}}
\theoremstyle{thmstyleone}%
\newtheorem{theorem}{Theorem}%  meant for continuous numbers
\newtheorem{proposition}[theorem]{Proposition}% 

\theoremstyle{thmstyletwo}%
\newtheorem{remark}{Remark}%

\theoremstyle{thmstylethree}%
\newtheorem{definition}{Definition}%

\begin{document}

\title[Spatial Graph Coarsening]{Topological Spatial Graph Coarsening}


\author[1]{\fnm{Anna} \sur{Calissano}}\email{a.calissano@ucl.ac.uk}

\author*[2]{\fnm{Etienne} \sur{Lasalle}}\email{etienne.lasalle@ls2n.fr}

\affil[1]{\orgdiv{Department of Statistical Science}, \orgname{University College London}, \orgaddress{\street{1-19 Torrington Place}, \city{London}, \postcode{WC1E 7HB}, \country{United Kingdom}}}

\affil*[2]{\orgdiv{Nantes Université, École Centrale Nantes, CNRS}, \orgname{LS2N, UMR 6004}, \orgaddress{\city{Nantes}, \postcode{F-44000}, \country{France}}}


\abstract{
% spatial graphs
Spatial graphs are particular graphs for which the nodes are localized in space (\eg, public transport network, molecules, branching biological structures).
In this work, we consider the problem of spatial graph reduction, that aims to find a smaller spatial graph (\ie, with less nodes) with the same overall structure as the initial one. In this context, performing the graph reduction while preserving the main topological features of the initial graph is particularly relevant, due to the additional spatial information.
Thus, we propose a topological spatial graph coarsening approach based on a new framework that finds a trade-off between the graph reduction and the preservation of the topological characteristics. In order to capture the topological information required to calibrate the reduction level, we adapt the construction of classic topological descriptors made for point clouds (the so-called persistent diagrams) to spatial graphs. This construction relies on the introduction of a new filtration called triangle-aware graph filtration.  
Our coarsening approach is parameter-free and we prove that it is equivariant under rotations, translations and scaling of the initial spatial graph. We test the method on synthetic and real spatial graphs (road network and fungi network), and show that it significantly reduces the graph sizes while preserving the relevant topological information. 
}

\keywords{Spatial graphs, Graph Coarsening, Topological Data Analysis}



\maketitle

\section{Introduction}\label{sec1}

% What are spatial graphs
Spatial graphs are a specific type of graphs with spatial attributes associated to the nodes and edges. They are studied in different applicative context and called with various names. In transport analysis, social network analysis, and epidemiology, they are often referred to as spatial networks \citep{barthelemy2011spatial,gou2021understanding,doytsher2010querying}. Spatial graphs are also closely linked to geometric graphs — which are graphs whose nodes lie in Euclidean space and whose edges are (possibly intersecting) straight lines \citep{cerny2005geometric}. When edges may instead be continuous curves, the structure is known as a topological or metric graph \citep{pach1997graphs}. Geometric, topological, and metric graphs are extensively studied within graph theory, which explores the possible configurations of a graph with a given number of nodes and constraints on edges; see \cite{pach2013beginnings,pach2004geometric} for historical and technical perspectives. Finally, spatial graphs can also be viewed as simplified forms of shape or elastic graphs, where actual shapes are encoded along the edges \citep{guo2020representations}. When working with such graphs, a common problem is dimensionality as they may have a high number of nodes and edges. Airways in the lungs or neurons networks are common real world examples displaying high dimensionality. The problem of decreasing the graph complexity has been studied in the network analysis literature under the name of graph sparsification or graph coarsenin, see \cite{hashemi2024comprehensive} for a recent overview of the field. Graph sparsification consists in selecting only a subset of nodes and edges \citep{althofer1993sparse}, while graph coarsening consists in grouping the nodes into super-nodes and edges into super-edges using aggregation algorithms \citep{loukas2018spectrally}.\\

% what do we want to do and why it is needed
In this paper, we propose a way to simplify a spatial graph by preserving its topology. We want to preserve the principal graph structure (connected components, large cycles) while removing small and noisy graph components. The principal adjective recalls the principal component analysis as our method is similarly trying to capture the characteristics of the graph which retain the majority of the information. We propose a spatial graph coarsening procedure which is topologically informed - here called Topological Spatial Graph Coarsening.  
To the best of our knowledge, the proposed method is new in the literature and addresses the open challenge in many real world applications of working with graphs with lower number of nodes. To define a dimensionality reduction procedure which is topologically informed, we need to estimate the topological properties of the spatial graphs. We leverage tools from Topological Data Analysis (TDA) \citep{chazal2021introduction}. Taking its root in the field of algebraic topology, TDA is a recent and promising domain that provides well-founded tools to analyze complex data - such as graphs -  for which topological properties represent key insights into understanding the data. Here, by topology we refer to the properties left unchanged by continuous deformations such as rotation, translation, stretching, and bending (\eg, number of connected components or holes). A common tool in TDA is the persistent diagram, which is classically used to study the topology of point clouds in $\mathbb{R}^p$ - see seminal works by \cite{edelsbrunner2002topological, carlsson2009topology, aktas2019persistence}. As for point clouds, some attempt have been done to adapt classic filtration to graph data \citep{archambault2007topolayout,li2012effective,ferrara2011topological}. For example, a graph filtration can be defined based on the node degree \citep{hofer2017deep}, the Jaccard index and Ricci curvature \citep{zhao2019learning}, or the Heat Kernel \citep{carriere2020perslay,lasalle2024heat}. As our graphs have a spatial component, we define a new filtration which is based on both the spatial and the structural information. We name it the triangle-aware graph filtration as it captures cycles information generated by all cycles, even the one induced by triangles in the graph. Given the topological characterization, we can now propose a coarsening procedure for spatial graphs. The new graph preserve the most important topological features of the original graph - captured by our filtration and persistent diagram -  with a lower number of nodes and edges. As for any dimensionality reduction technique, we would like to calibrate the reduction level, as too much reduction would yield to an oversimplified graph with missing important features, and too little reduction would not provide a significant and useful coarsening. We define a score to guide the choice of an appropriate reduction level. This score balances the reduced graph complexity with the amount of topological ``distortion'' between the original and reduced graph evaluated by comparing persistence diagrams. Minimizing this score allows to calibrate the reduction level.\\

% From one to a set
%The paper focuses both on dimensionality reduction for one and for a set of spatial graphs. There is an increase interest in working with set of graphs, due to the high number of applications in which such sets arise. At a population level, the dimensionality reduction technique focuses on finding the optimal simplified graph per each graph in the set. Such optimal simplified graph is found by leveraging the topological information of the set via the persistent diagram of the population mean and the amount of edge removed in the specific graph (to avoid overcoarsenings). Defining a coarsening based on the topological property of the population is a more robust way of selecting the key topological features, based on the underlying hypothesis that the spatial graphs share common characteristics.

The paper is organized as follow. In Section \ref{sec:graph_coarsening}, we introduce the spatial graph object and the coarsening procedure. In Section \ref{sec:PD}, we recall how to compute classic persistence diagrams and we introduce the persistence diagram adaptation for spatial graphs (Subsection \ref{sec:adaptation_for_graphs}). In Section \ref{sec:topo_graph_coarsening}, we define the topologically informed spatial graph coarsening. As the new methodology is tailored for spatial graphs, we study the invariance and equivariance properties of the method with respect to rotation, translation and scaling of the node positions; see Section \ref{sec:properties}. In Section \ref{sec:experiments}, we showcase the proposed method on simulated data and on a real world dataset of fungi networks \citep{fricker2025fungi}. The Python code of our method is publicly available in GitHub \citep{github_repo}.

\section{Spatial Graph Coarsening}
\label{sec:graph_coarsening}

 A spatial graph $G$ is given by a set of nodes $V=\{v_1, \dots, v_n\}$, a set of edges $E\in V\times V$ and spatial positions of the nodes $(X_u)_{u \in V}$, where $x_u \in \R^p$ for all nodes $u$, with $p$ being the space dimension. The spatial graph is denoted by the triplet $G=(V,E,X)$. In this work, the nodes are embedded in the space $\mathbb{R}^p$ equipped with standard Euclidean norm $\|\cdot\|$. In classic applications, the dimension $p$ is usually 2 or 3. The length function $\ell:E\rightarrow\mathbb{R}_+$ assigns to each edge $(u,v) \in E$ its edge length $\ell_{u,v} = d(x_u,x_v)$, computed via Euclidean distance. \\

The objective of this paper is to define a coarsening procedure for spatial graphs, retaining the main topological properties of the original graph. In a general graph setting, given a set of graphs $\mathcal{G}$, a coarsening method can be defined via a coarsening function $f:\mathcal{G}\rightarrow\mathcal{G}$ which defines a new graph on a partition of the nodes of the original graph: $f(G)=(f^V(G),f^E(G))$ where $f^V(G)\in \mathcal{P}(V)$, with $\mathcal{P}(V)$ denoted the set of partitions of $V$. As our graphs are spatial, we define a spatial graph coarsening methods based on a spatial parameter $\theta \geq 0$. Such parametrization of the coarsening procedure allows to define a coarsened graph which depends on both the structural and the spatial structure of the graph. Essentially, we obtain the coarsened spatial graph by collapsing the edges with length smaller than $\theta$. 
Before giving the formal definition of our spatial graph coarsening (Definition~\ref{def:spatial_graph_coarsening}), we introduce the notation $sub(G, \theta) = (V, E_\theta)$, such that $E_\theta = \{ (u,v) \in E, \ s.t. \ \ell_{u,v} \leq \theta \}$. $sub(G, \theta)$ denotes the abstract subgraph of $G$, induced by the edges of length smaller than $\theta$.

\begin{definition}
    Consider a spatial graph $G=(V,E,X)$ and a parameter $\theta \geq 0$. We define the coarsened graph at scale $\theta$ by $f_\theta(G) = (f_\theta^V(G), f_\theta^E(G), f_\theta^X(G))$ where the three components $f_\theta^V(G)$, $f_\theta^E(G)$ and $f_\theta^X(G))$, acting respectively on the nodes, edges and positions, are defined as follows:
    \begin{itemize}
        \item the set of nodes of the coarsened graph $f_\theta^V(G) \in \mathcal{P}(V)$ is defined as the partition of $V$ that corresponds to the connected components of the subgraph $sub(G, \theta)$. Each element of this partition is called an hypernode, as it is a subset of nodes in $G$.
        \item the set of edges $f_\theta^E(G) \subset f_\theta^V(G) \times f_\theta^V(G)$ is obtained by including the edge $(V_i,V_j)$ in $f_\theta^E(G)$ if and only if there exist $u\in V_i$ and $v\in V_j$ such that $(u,v)\in E$.
        \item we proposed two strategies to assign the new node positions, the average positioning and the degree positioning:
        \begin{itemize}
            \item the average positioning assigns to an hypernode $V_i$ the average position of the merged nodes:
$$f_\theta^X(G)_i=\frac{1}{|V_i|}\sum_{u\in V_i}x_u.$$ 
            \item  the degree positioning assigns to an hypernode $V_i$ the position of the node with the highest degree among the merged nodes:
$$f_\theta^X(G)_i=x_u, \quad u=argmax_{v\in V_i}deg(v)$$
where the degree $deg(v)\in \mathbb{R}_+$ is the number of edges incident to $v$. 
        \end{itemize}
    \end{itemize}
    \label{def:spatial_graph_coarsening}
\end{definition}

Given the new set of nodes, the new set of edges and the new coordinates, the coarsened graph $f_\theta(G)$ can be equipped with the length function $\ell_{f_\theta(G)}$ from $f_\theta^X(G)$ using the Euclidean distance between hypernodes. The choice of which position strategy to adopt (average or degree positioning) depends on the application at hand. Degree positioning is often preferable for physical networks. For example, a road network - where the nodes represents the crossing and the edges represents the roads - can be simplified by excluding minor crossing but the position of the remaining crossing should be preserved. The average positioning might be a better option when the graph represents a skeleton of an object. The coarsening procedure is applied to simplify the network yet preserving the object topology.

\begin{remark}
    The parameter $\theta$ takes values in $\mathbb{R}_+$. However, in practice, the interesting set of values for the parameter is $\Theta=[\min_{(u,v)\in E} \ell_{u,v}, \max_{(u,v)\in E} \ell_{u,v} ]$. The coarsened graph for $\theta> \max_{(u,v)\in E} \ell_{u,v} $ is the trivial graph where all the nodes are collapsed into one single node. To simplify notations, we continue to consider $\theta$ in $\mathbb{R}_+$
\end{remark}

Once we have introduced a spatial coarsening procedure, the question is now how to select the parameter $\theta$. The objective of the paper is to define a coarsened graphs with lower number of nodes while preserving the key topological features of the original graph. An illustrative example is shown in Figure \ref{fig:pruned_graphs}. We sample $20$ nodes at random with position on an annulus shape. The corresponding fully connected graph has $(n^2-n)/2$ edges, but we select only a percentage of the edges, keeping only those with smallest length. The original graph is reduced into three graphs with lower number of nodes, obtained with different values for the parameter $\theta$. From the picture it is clear to see how the coarsening procedure is first merging the nodes while preserving its key topological feature which is largest the cycle. As the parameter $\theta$ increases, more nodes are merged together, resulting in an oversimplified graph. To identify a parameter $\theta$ that performs a good balance between reducing the graph and preserving the main topological features, we define a procedure that relies on persistence diagrams to capture the topological information.
\begin{figure}[H]
    \centering
    \includegraphics[width=0.9\linewidth]{figures/examples_reduced_annulus_graphs.pdf}
    \caption{Examples of pruned graphs for different values of $\theta$. The leftmost graph is the initial graphs.}
    \label{fig:pruned_graphs}
\end{figure}

\section{Persistence Diagrams for Spatial Graphs} \label{sec:PD}

classic TDA pipelines consist in capturing certain properties for shapes underlying the data such as number of holes or number of connected components. One of the main tools is a topological descriptor called persistence diagram. In this section, we briefly introduce the general TDA pipeline for analyzing point clouds, before describing how they can be adapted to graph, and especially spatial graphs.  For a more complete presentation of TDA, we refer the reader to \cite{edelsbrunner2010computational},\cite{oudot2015persistence}, \cite{chazal2021introduction}, and \cite{iniesta2022topological}.

\subsection{classic persistence diagrams}

Data as point clouds have no intrinsic interesting topology (only isolated contractible connected components). However, there often exist underlying continuous structures with richer topology, \eg, points sampled on some manifolds. To capture the topological invariants of the underlying shape, we construct continuous spaces on top of the point clouds called simplicial complexes. A simplicial complex is a space made by gluing together simple structures of different dimensions: points, segments, triangles, tetrahedrons etc. The simple building blocs are called simplices. They are represented in an abstract way as a set of their vertices, \eg, an segment between point $u$ and point $v$ is encoded as $\{u,v\}$, a triangle between $u$, $v$ and $w$ is given by $\{u,v,w\}$. Simplicial complexes are a generalization of classic graphs, which are one-dimensional objects (composed of points and edges only), where one can include higher-dimensional objects like triangles (dimension 2), tetrahedron (dimension 3) and equivalents in higher dimensions. Formally, they are a set of simplices, \ie, a set of sets of vertices.  Given a point clouds $\Xcal = \{x_1, \dots x_n\}$ in $\R^p$ equipped with the Euclidean distance $d$, a classic simplicial complex built on top of $\Xcal$ is the so-called Vietoris-Rips (VR) complex at scale $r > 0$ denoted by $VR_r(\Xcal, d)$ \citep{zomorodian2010fast}. 
To determine which simplices are included in $VR_r(\Xcal)$ consider all balls of radius $r$ centered on the points of $\Xcal$. If two balls intersect, the corresponding edge between the two centers is an element of $VR_r(\Xcal, d)$. More generally, for higher dimension, any $k$-dimensional simplex (subset of $k+1$ points of $\Xcal$) is included if all pairs of balls centered in points of the simplex intersect. Formally, the $k$-simplex $\{ x_{i_1}, \dots, x_{i_k} \}$ is in $ VR_r(\Xcal, d)$, if for all $j$ and $j'$ such that $1 \leq j < j' \leq k$, it satisfies $d(x_{i_j}, x_{i_{j'}}) \leq r$.\\

% \subsection{Persistence Diagrams}
From this continuous structures one can now compute topological invariants of the data \citep{chazal2021introduction}. 
To extract meaningful topological information from the data, instead of looking at the topological invariants of one simplicial complex at a given scale $r > 0$, we can look at their evolution in the family of simplicial complexes $(VR_r(\Xcal, d))_{r > 0}$ as $r$ increases.  Such a family is called a filtration, \ie, a family of nested simplicial complexes. 
As $r$ increases, some topological features (connected components, holes, cavities, etc.) appear and disappear. For each topological feature, we record the scale $r_b$ at which it appears (it is born) and the scale $r_d \geq r_b $ at which it disappears (it dies). We encode this information into a point $(r_b, r_d)$ in $\R^2$. The collection of such points for all topological features is called the persistence diagram (PD). 
Points that are closer to the diagonal $y=x$ in the PD corresponds to topological features that lived during a short range of scales, they did not persist, hence they are usually associated to topological noise. Conversely, points that lie far from the diagonal correspond to topological features that persisted in the filtration and therefore correspond to important topological characteristics of the data. 
Note that monitoring the topological invariants in the filtration avoids having to choose a specific scale at which relevant topological features would be present. How much a topological feature persists in the filtration encode its relevance as a topological information of the data.\\

The space of persistence diagrams can be equipped with the so-called Bottleneck distance \citep{cohen2005stability} denoted by $\db$ and defines as follows. Let $\mu$ and $\nu$ be two PDs (\ie, sets of points in $\R^2$), $\Delta := \{(a,a), a \in \R\}$ be the diagonal and $\Pi(\mu, \nu)$ be the set of bijections from $\mu \cup \Delta$ to $\nu \cup \Delta$. Then, $\db$ is defined as 
\begin{equation}
    \db(\mu, \nu) = \underset{\pi \in \Pi(\mu, \nu)}{\inf} \underset{x \in \pi(x)}{\sup} \|x - \pi(x)\|_\infty \,.
    \label{eq:bottleneck_distance}
\end{equation}
Considering the bijections between the PDs augmented by the diagonal allows to compare PDs with different numbers of points, as extra points can be matched with points on the diagonal.

\subsection{Persistence Diagram for Spatial Graphs \label{sec:adaptation_for_graphs}}

We want to adapt the above construction of persistence diagrams of point clouds data to spatial graphs data. To do so, we follow the same general idea: build a nested family of simplicial complexes on top of graphs and monitor how the topological components evolve through this filtration. Previous approaches for computing filtrations over graphs considered a family of nested subgraphs of $G$, $(G_r)_{r\geq 0}$ (by filtering the edges based on weights or other graph characteristics like degree) and compute simplicial complexes over each subgraph $C(G_r)$ where $C$ can be for example the clique complex (also called flag complex) or the neighborhood complex \citep{lovasz1978kneser, aharoni2005eigenvalues, horak2009persistent}. These approaches differ from the point cloud pipeline on different aspects. First, the scale is given by a substructure of the initial object (here the graph) while for point clouds we consider the whole dataset and compute simplicial complexes at different scales. Secondly, some simplices are never included in the filtration (\eg, a 1-simplex that correspond to an edge that is not in $G$ will never be include in the clique complex).\\

We propose a new way of computing a filtration of simplicial complexes over a graph that is closer to the construction for point clouds. The essence of our approach lies in the fact that the VR filtration is entirely determined by the distances between all pairs of points in the data. Thus, we propose to replace the Euclidean distance used in point clouds by the shortest path distance over the graph, denoted by $d_G$ and defined for all pairs of vertices $u$ and $v$ by 
\begin{equation}
    d_G(u,v) = \min_{k\geq 2} \min_{\substack{u_1, \dots, u_k \in V, \\ \text{s.t. } u_1 = u, \ u_k = v, \\ \forall 1 \leq i \leq k-1, \ (u_i, u_{i+1}) \in E}} \sum\limits_{i=1}^{k-1} \ell_{u_i, u_{i+1}},
    \label{eq:shortest_path_dist}
\end{equation}
where we recall that the length $\ell_{u,v}$ of an edge $(u,v)$ is given by the Euclidean distance $d(x_u, x_v)$. Therefore, we define a simplicial complex at scale $r>0$ over the graph $G$ as $\mathcal{C}_r(G) = VR_r(V, d_G)$. 

\begin{remark}
    If the graph is the complete graph, the shortest path distance coincides with the Euclidean distance and we recover the classic $VR_r$ filtration.
\end{remark}

The proposed filtration $\mathcal{C}(G) = (\mathcal{C}_r(G))_{r \geq 0}$ has the drawback that cycles formed by a triplet of edges (empty triangle) do not persist in the filtration. In fact, the 2-simplex (filled triangle) appears in the filtration exactly when the longest edge is included, making the cycle disappear at the same scale at which it appears. Thus, this type of cycles are not represented in the induced persistence diagram. As we would like to represent in the persistence diagram any cycles - no matter the number of edges - we propose to modify the construction of the simplicial complexes $\mathcal{C}_r(G)$ by enlarging the scale at which 2-simplices appear.  This novel filtration is defined below and illustrated in Figure \ref{fig:triangle_aware_filtration} in comparison with the classic Vietoris-Rips filtration.

\begin{definition}[Triangle-aware graph filtration]
    Given a spatial graph $G$, we define the new triangle-aware graph filtration $\widetilde{\mathcal{C}}(G) = (\widetilde{\mathcal{C}}_r(G))_{r \geq 0}$. For each scale $r \geq 0$, the simplicial complex $\widetilde{\mathcal{C}}_r(G)$ is the same as  $\mathcal{C}_r(G)$, except that 2-simplices are now include for larger values of $r$. A given 2-simplex $\{u,v,w\}$ appears in $\widetilde{\mathcal{C}}(G)$ at scales larger or equal to $\min \left\{ d_G(u,v)+d_G(v,w) , \ d_G(u,w)+d_G(v,w) , \ d_G(u,v)+d_G(u,w) \right\}$. Finally, to turn the filtration $\widetilde{\mathcal{C}}(G)$ into a valid filtration in any dimension, we modify the scale at which simplices of larger dimension appear, making sure that they appear after their faces (sub-simplices). To do so, the scale at which a $k$-simplex for $k \geq 3$ appears is set to the maximum of the scales at which its faces appear.  
\end{definition}

\begin{figure}
    \centering
    \includegraphics[width=\linewidth]{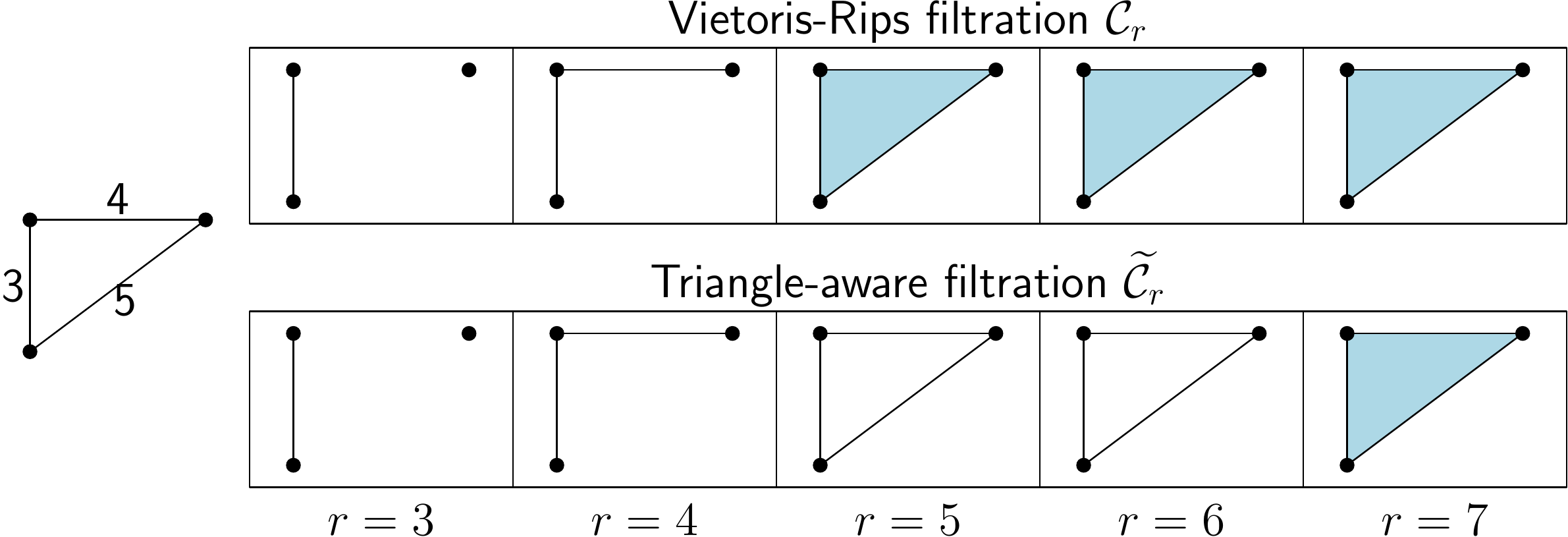}
    \caption{Illustration of the difference between the classic Vietoris-Rips filtration (top) and the proposed triangle-aware filtration (bottom), where the 2-simplex appears at scale $r=5$ and $r = \min\{3+4, \ 4+5, \ 3+5\}=7$, respectively.}
    \label{fig:triangle_aware_filtration}
\end{figure}

Note that in  $\mathcal{C}(G)$ the 2-simplex $\{u,v,w\}$ appears at scales larger or equal to $\max \left\{ d_G(u,v), d_G(v,w), d_G(u,w) \right\}$, so the triangle-aware filtration $\widetilde{\mathcal{C}}(G)$ actually delays the appearing of the 2-simplices. Note also that even in dimension 2, the filtration is a valid filtration as, from the triangular inequality, the triangle appears after its edges. Remark also that with this filtration, flat triangles yields cycles with small persistence while equilateral triangles maximize the persistence.

In Figure \ref{fig:original_G_and_PD}, we show a persistent diagram of a spatial graph $G$. Red points correspond to topological features of dimension 0 (\ie, connected components) while the blue points correspond to topological components of dimension 1 (\ie, cycles).
% \begin{figure}[H]
%     \centering
%     \includegraphics[width=0.7\linewidth]{figures/inital_annulus_graph_and_PD.pdf}
%     \caption{Original graph and its associated persistent diagram.}
%     \label{fig:pd_init}
% \end{figure}

\begin{figure}[htp]
\centering
    \includegraphics[width=0.53\linewidth]{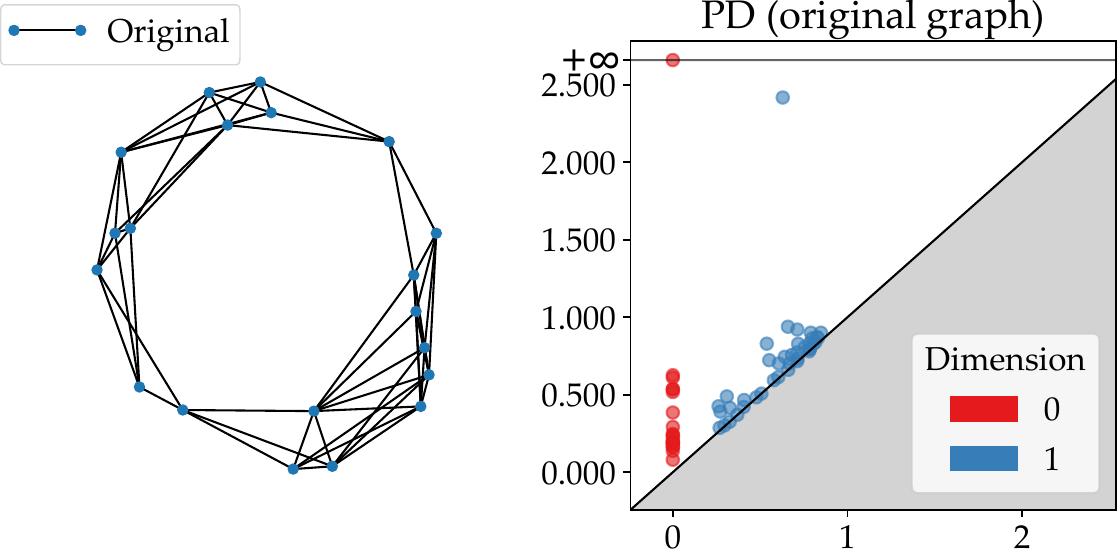}
    \caption{A graph (left) and the corresponding persistence diagram (right). }
    \label{fig:original_G_and_PD}
\end{figure}

Using the bottleneck distance, these topological descriptors can now be used to compare the topology of a given graph and and the one of its coarsening version, to quantify the topological distortion caused by the coarsening procedure. We exploit these tools in the next section.

\section{Topological Spatial Graph Coarsening \label{sec:topo_graph_coarsening}}

Given a sequence of coarsened graphs $f_\theta(G)$ obtained by exploring the set $\Theta$, we define a scoring function to select the optimal coarsened graph. In graph coarsening literature, different scoring functions have been defined: from reconstruction methods where the score represents the capability of reconstructing the original graph (see for example spatial coarsening \cite{lefevre2010grass,riondato2017graph} and spectral coarsening \cite{kumar2023featured}); to reconstruction-free method where the scoring function is measuring interesting features for the specific application (for a review see \cite{hashemi2024comprehensive}). We define a scoring function that accounts for the balance between the topological characteristics distortion and the complexity of the coarsened graph.
\begin{definition}[Topology-informed scoring function]
Given an original spatial graph $G=(V,E,X)$ and a coarsened graph $f_\theta(G)=(f^V_\theta(G), f_\theta^E(G), f_\theta ^X(G))$. Given the corresponding persistent diagrams $PD(G),PD(f_\theta(G)) \in \mathcal{PD}$, where $\mathcal{PD}$ is the space of persistent diagrams. The scoring function is:
\begin{equation}
    S_\theta(G)= \frac{|f_\theta^E(G)|}{|E|} + \lambda \cdot  \db(PD(G), PD(f_\theta(G))) \,,
    \label{eq:score}
\end{equation}
where $\db:\mathcal{PD}\times\mathcal{PD}\rightarrow\mathbb{R}$ is the bottleneck distance between persistent diagrams, defined in \eqref{eq:bottleneck_distance} and $\lambda\in \mathbb{R}$ is a scaling parameter. 
\end{definition}
The scoring function is a combination of the proportion of remaining edges in the graph and the topological difference between the original and the reduced graph, measured via the bottleneck distance between persistent diagrams. To match the order of magnitude of both terms, $\lambda$ is set to the inverse of the maximal distance between the persistence diagrams of the initial and a reduced graph giving $\lambda(G)^{-1} := \max_\theta \db \left( PD(G), PD(f_\theta(G))) \right)$. Our topological spatial graph coarsening method minimizes the score with respect to $\theta$ to find a satisfying reduction level. It outputs the coarsened graph corresponding to this optimal $\theta$. In Figure \ref{fig:original_reduced_G_and_PD}, we show a spatial graph and its corresponding reduced graph obtained with our topological spatial graph coarsening (here $\theta^*=0.4$) using average node positioning. We also display the persistent diagrams of both graphs, showing that the persistent diagram of the reduced graph displays the main topological features (points far away from the diagonal) of the original one, while most of the ``topological noise'' (points close to the diagonal) have been removed, demonstrating the effect of the graph simplification. 

\begin{figure}[H]
    \centering
    \includegraphics[width=0.9\linewidth]{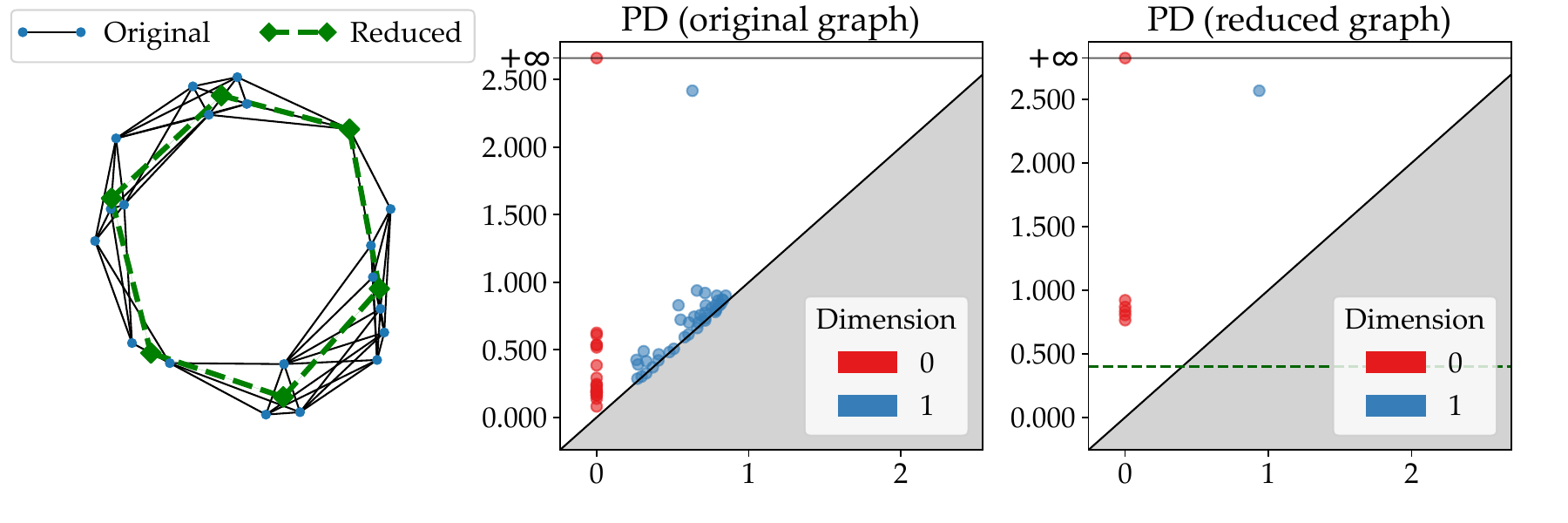}
    \caption{From left to right. The original and reduced graph. The persistence diagram of the original graph. The persistence diagram of the reduced graph; horizontal dashed line represents the threshold $\theta$ used to compute the reduced graph.}
    \label{fig:original_reduced_G_and_PD}
\end{figure}

\section{Property of Spatial Graph Coarsening}\label{sec:properties}

We often want coarsening methods to be coherent and robust to certain modifications of the initial graph. In this work, we show that our topological spatial graph coarsening is equivariant with respect to certain group actions. In particular, as we are dealing with nodes embedded in Euclidean space, the actions of interest are rotation, translation, reflection, and scaling.  Let's start by defining how these groups act on the spatial graph.
\begin{definition}
    Consider a spatial graph $G=(V, E, X)$ with coordinated in $X\in \mathbb{R}^{p\times n}$, $|V|=n$, and the similarity group $Sim(p)$. We can define a set of equivalent graphs $[G]=\{\psi(G,(R,A,k))=(V,E,kRX+A), (R,A,k)\in Sim(p)\}$ by applying the group action via $\psi: \mathcal{G}\times Sim(p) \rightarrow \mathcal{G}$, where $R\in O^p(\mathbb{R})$ is an orthogonal matrix, $A\in \mathbb{R}^p$ is a vector translating the graph, and $k>0$ is a strictly positive constant value scaling the graph.
\end{definition}
In the paper, we focus on graphs with nodes embedded in $\R^p$ with $p=2,3$ but the results hold in higher dimension. The symmetric group $Sym(p)$ is one of the common group studied in shape and image analysis. Its elements $(R,A,k)$ can rotate, scale, and translate the original object. In the case of spatial graphs, the equivalence class of the original graph $G$ contains all the graphs that share the same set of nodes $V$, the same connectivity structure $E$, and a similarity transformation of the spatial coordinated of the nodes in $\mathbb{R}^p$ (i.e. a set of rotated, scaled, and translated coordinates). The objective of this section is to study the properties of the spatial graph coarsening procedure. While the writing might seems a bit cumbersome, the idea behind are simple and represented in an illustration in Figure \ref{fig:properties_illustration}. We want to show how the procedure is retrieving the same graph yet rotated, translated and scaled when an element of the symmetric group is applied to the original spatial graph. To do so, we proceed step by step. Firstly, we look at the spatial graph coarsening and we prove the parameter equivariance property (Section \ref{sec:par_equiv}). Secondly, we look at the topological spatial graph coarsening. To prove its invariance under the group action, we need to look at how the action effect the scoring function (Section \ref{sec:action_scoring}). Finally, we prove the main result of this section, Proposition~\ref{prop:equivariance_final}, that shows the equivariance of our topological spatial coarsening procedure under the symmetric group action (Section \ref{sec:optimal_coarsening_prop}).

\begin{figure}
    \centering
    \includegraphics[width=0.29\linewidth]{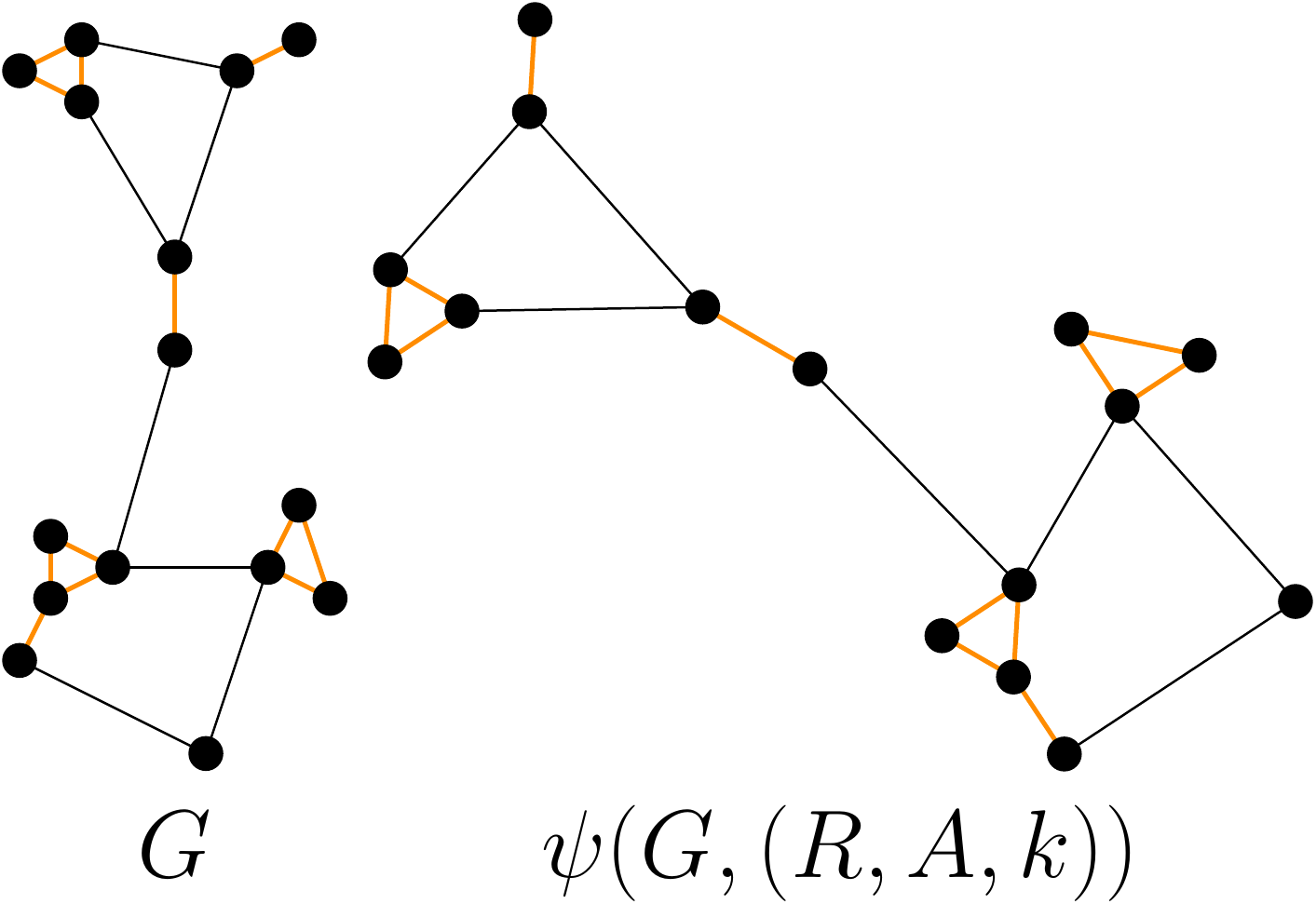}
    \hspace{0.03\linewidth}
    \includegraphics[width=0.29\linewidth]{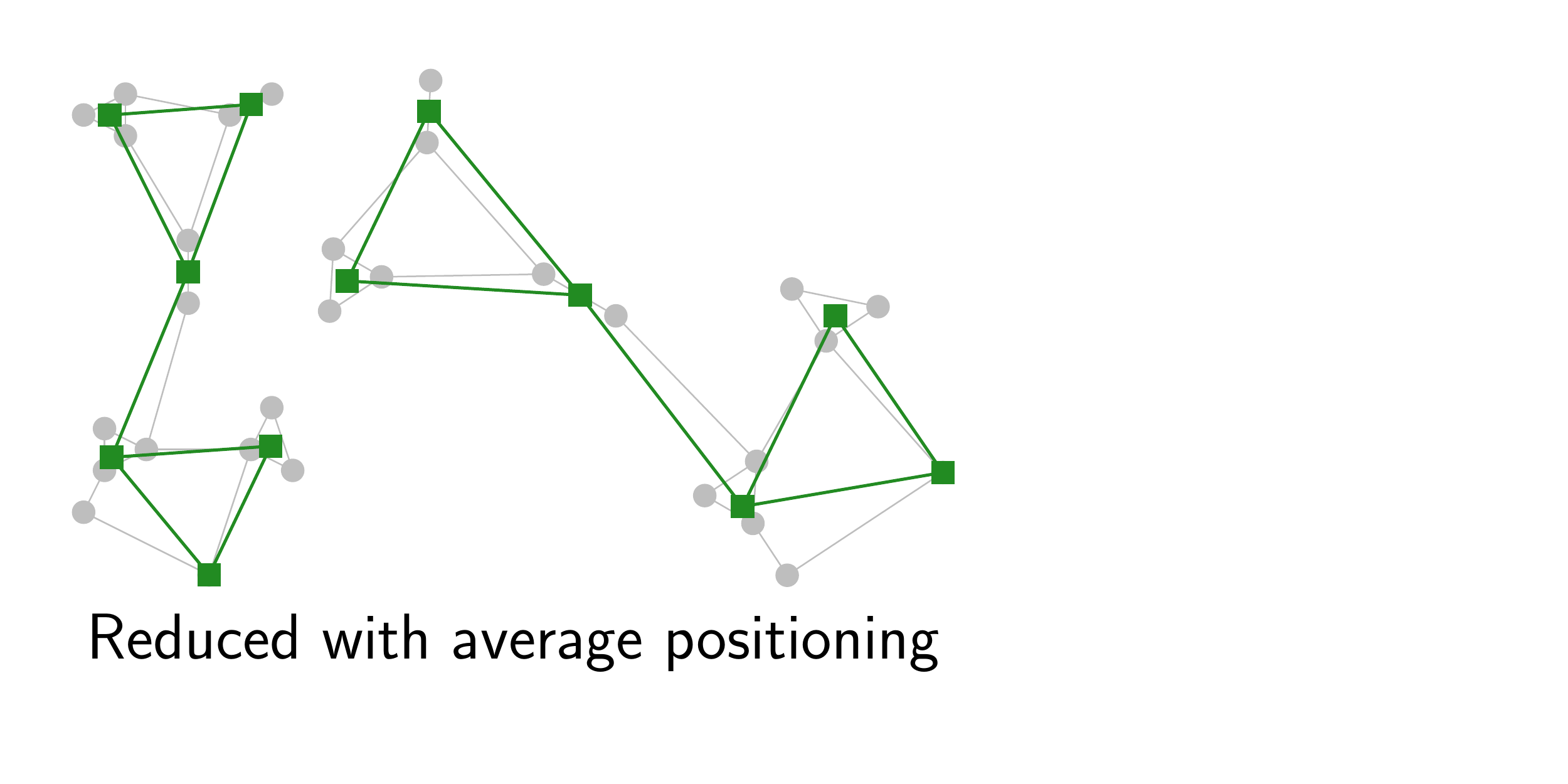}
    \hspace{0.03\linewidth}
    \includegraphics[width=0.29\linewidth]{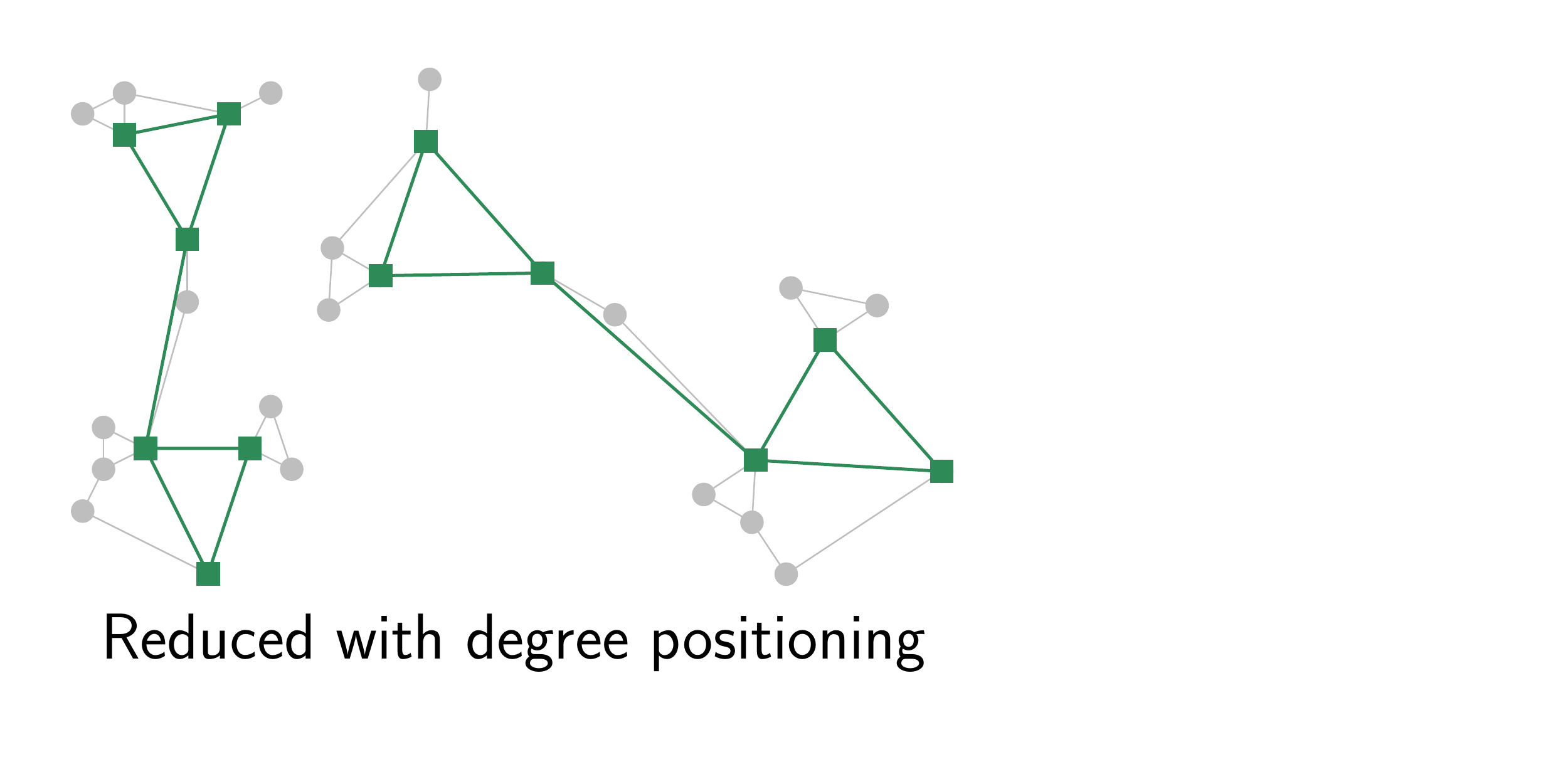}
    \caption{Illustration of the property of the coarsening procedure when a spatial graph is modified by applying a $(R,A,k)\in Sym(2)$}
    \label{fig:properties_illustration}
\end{figure}

\subsection{Parameter equivariance property for the spatial graph coarsening}\label{sec:par_equiv}
We start by looking at the spatial graph coarsening described in Section \ref{sec:graph_coarsening}. We need to define how the group action is acting on the parameter $\theta$ of the coarsening procedure:
\begin{definition}
    Given a spatial coarsening function $f_\theta:\mathcal{G}\times\mathcal{G}\rightarrow\mathcal{G}$ and an element of the symmetric group $(R,A,k)\in Sim(p)$, we define $\phi(\theta,(R,A,k))=\phi(\theta,k)=k\theta$ and the scaled spatial coarsening function $f_{\phi(\theta,k)}$.
\end{definition}
As a first step, we prove that the spatial graph coarsening is parameter equivariant:
\begin{proposition}\label{prop:coarsening}
    The spatial coarsening is parameter equivariant under the similarity group action:

    $$\psi(f_\theta(G),(R,A,k))=f_{\phi(\theta,k)}(\psi(G,(R,A,k))) \,.$$
\end{proposition}
\begin{proof}
%Consider a graph $G=(V,E,X)$ with nodes $X\in\mathbb{R}^p, p=2,3$. Consider a value $\theta\in[min_{(u,v)\in E}(\ell_{u,v}),max_{(u,v)\in E}(\ell_{u,v})]$, the coarsening procedure introduced in Section \ref{sec:graph_coarsening} consist in finding an value $\theta$ to obtain a coarsened graph $f_\theta(G)=(f^V_\theta(G),f^E_\theta(G),f^X_\theta(G))$. As the coarsening procedure depends only on the edge length $\ell_{u,v}, (u,v) \in E$, it is trivial to see that the procedure is invariant under rotation and translation (as they left the edges' length unchanged) and it is equivariant under scaling when the parameter $\theta$ is properly rescaled. Consider $(k,R,A)\in Sim(p)$ where the scaling factor is $k\in \mathbb{R}_+$, the resulting graph $G'=(V,E,kRX+A)$ has edges' lengths $\ell'_{u,v}=k \ell_{u,v}$. We want to prove that $$\psi(f_\theta(G),(R,A,k))=f_{\phi(\theta,k)}(\psi(G,(R,A,k)))$$.
Consider a graph $G=(V,E,X)$ with node positions $X\in\mathbb{R}^p, p=2,3$ and a length threshold $\theta\in \R_+$. The coarsening procedure introduced in Section \ref{sec:graph_coarsening} that defines the coarsened graph $f_\theta(G)=(f^V_\theta(G),f^E_\theta(G),f^X_\theta(G))$ is governed by the edge lengths of the original spatial graph $G$. It is trivial to see that the coarsening procedure is invariant under rotation and translation, as they are isometries and thus they leave the edge lengths unchanged. The coarsening procedure is also equivariant under scaling, if the parameter $\theta$ is properly rescaled. More precisely, let us consider $(k,R,A)\in Sim(p)$ where the scaling factor is $k\in \mathbb{R}_+$, the resulting graph $G'=(V,E,kRX+A)$ has edge lengths $\ell'_{u,v}=k \ell_{u,v}$. We want to prove that $$\psi(f_\theta(G),(R,A,k))=f_{\phi(\theta,k)}(\psi(G,(R,A,k)))$$.

To do so, we need all the three components of the coarsening function to respect the property:
\[ %\psi(f_\theta(G),(R,A,k))=\psi(f_\theta^V(G),f_\theta^E(G),f_\theta^X(G),(R,A,k))=\]\[
(f_\theta^V(G),f_\theta^E(G),kR(f_\theta^X(G))+A)=%f_{\phi(\theta,k)}(\psi(G,(R,A,k)))=\]\[f_{\phi(\theta,k)}((V,E,kRX+A))=
(f^V_{\phi(\theta,k)}(G'),f^E_{\phi(\theta,k)}(G'),f^X_{\phi(\theta,k)}(G')) \,,
\]
where we recall that $G' = \psi(G,(R,A,k)) = (V,E,kRX+A)$.

Let's start from the node function and prove $f_\theta^V(G)=f_{\phi(\theta,k)}^{V}(G')$.  Let us recall that $sub(G, \theta)= (V, E_\theta)$ is the combinatorial (not spatial) subgraph induced by the edges of length smaller that $\theta$, such that $E_\theta = \{ (u,v) \in E, \ s.t. \ \ell_{u,v} \leq \theta \}$. Knowing that $G'=(V,E,kRX+A)$ with edge lenghts $\ell'_{u,v}=k\ell_{u,v}$ it is easy to see that $sub(G,\theta)=sub(G',k\theta)$. Thus, $f_\theta^V(G)=f_{\phi(\theta,k)}^{V}(G')$ is the same partition of nodes obtained from the same connected components of $sub(G,\theta)$.

The function over the edges follows naturally, as the edges in the coarsened graph are defined from the hypernodes, so $f^E_\theta(G)=f^E_{k\theta}(G')$.

As last step, we need to prove the equality on the coordinates $f_{k\theta}^X(G') = kRf_\theta^X(G)+A$. This third element of the coarsening function $f_\theta^X:\mathbb{R}^p\rightarrow \mathbb{R}^p$ assign position with either degree or average positioning. For the average positioning of $V_i, i=1,\dots,K$:
 \[  f_\theta^X(G)_i = \frac{1}{\left| V_i\right|} \sum\limits_{u \in V_i} x_u. \]
As the partition of the nodes is the same, we can write for all hypernode number $i$:

     \[  f_{k\theta}^X(G')_i = \frac{1}{\left| V_i\right|} \sum\limits_{u \in V_i} x'_u=\frac{1}{\left| V_i\right|} \sum\limits_{u \in V_i} (kRx_u+A) = kR\frac{1}{\left| V_i\right|}\sum\limits_{u \in V_i} x_u+A=kRf_\theta^X(G)_i+A \,. \]
    
    Similarly, the degree positioning assign to the new hypernode $V_i,i=1,\dots,m$ is defined by searching for $u = \arg\max_{v \in V_i} deg(v)$. As the connectivity structure of the graph $G'$ is the unchanged, the degree is also the same. We then conclude that: 
    \[  f_{k\theta}^X(G')_i = x'_{u} = kRx_u+A=kRf_\theta^X(G)_i+A \,. \]
    As all the three elements of the function are equal we can conclude that the coarsening function is parameter equivariant:
    \[\psi(f_\theta(G),(R,A,k))=f_{\phi(\theta,k)}((V,E,kRX+A)) \,.\]
\end{proof}
The relationships between the elements in Proposition \ref{prop:coarsening} are detailed in the diagram below:\\

\begin{tikzpicture}[>=latex]
  % Nodes
  \node (G) {$G=(V,E,X)$};
  \node (G') [right=6cm of G] {$G'=(V,E,kRX+A)$};
  \node (Gtheta)  [below=1.5cm of G] {$f_{\theta}(G)$};
  \node (G'theta)  [below=1.5cm of G']{$f_{\phi(\theta,k)}(G')$};

  % Arrows
  \draw[->] (G) -- (G') node[midway, above] {$\psi(G,(R,A,k))$};
  \draw[->] (Gtheta) -- (G'theta) node[midway, above] {$\psi(f_\theta(G),(R,A,k))$};
  \draw[->] (G) -- (Gtheta) node[midway, left] {$f_\theta(G)$};
  \draw[->] (G') -- (G'theta) node[midway, right] {$f_{\phi(\theta,k)}(G')$};
\end{tikzpicture}

\begin{remark}
    The coarsening function is fully equivariant under the Euclidean group actionn $\mathbb{E}(d)$, where there is no scaling component. 
\end{remark}

\subsection{Impact of the similarity transformation on the persistent diagrams.}\label{sec:action_scoring}

Once we have proved the property of the coarsening procedure defined in Section~\ref{sec:graph_coarsening}, we would like to prove the equivariance property for the proposed topological spatial coarsening procedure defined in Section~\ref{sec:topo_graph_coarsening}. For that, we first prove that the persistent diagram is invariant under rotation and translation of the object and it is equivariant under scaling. 
\begin{proposition}\label{prop:pd}
    Consider two spatial graphs $G=(V, E, X)$ and $G'=\psi(G,(R,A,k))=(V,E,kRX+A)$, where $(R,A,k)\in Sim(p)$. The persistent diagrams: $$PD(G')=k\cdot PD(G)$$ 
    meaning that $(kr_b,kr_d)\in PD(G')$ if and only if $(r_b,r_d)\in PD(G)$.
\end{proposition}
\begin{proof}
    The invariance under rotation and translation has been proven in the literature of persistence diagrams \citep{chazal2009proximity,chazal2014persistence}; the equivariance under scaling follows from the definition of persistence diagrams in the context of spatial graphs. 
    More precisely, we want to show that that each simplex appearing in the filtration of the original graph $G$ at $r_b$, it appears at $k r_b$ in the filtration of the modified graph $\psi(G, (R,A,k))$. Therefore, each topological feature is born and died at scaled time, yielding the result. Therefore, we need to prove the result on simplices. 
    In the following, let us consider the action of a similarity with only a scaling part, that is $G' = \psi(G, (Id, 0, k))$, where $k > 0$, as we proved that the isometry part (rotation and translation) leaves the persistent diagram unchanged. 
    Therefore, the scaled coordinates are given by $X' = k X$ where $X$ are the Euclidean coordinates of the original graph $G$. Recall that the filtration $\mathcal{C}(G) = (\mathcal{C}_r(G))_{r \geq 0}$ is defined as the Vietoris-Rips filtration where the metric is given by the shortest-path distance over the graph (see Section~\ref{sec:adaptation_for_graphs}). From the definition of this filtration, for each scale $r \geq 0$ we have that $\mathcal{C}_{kr}(G')=\mathcal{C}_{r}(G)$, as all shortest-path distances in $G'$ are the same as the ones in $G$ except that they are scaled by $k$.
    In the specific case of spatial graphs, in Section~\ref{sec:adaptation_for_graphs}, we advocated the use of a modified filtration $\tilde{\mathcal{C}}$ to better capture cycles generated by triangles in the graph. To transpose the previous result to this modified filtration, we need to study the filtration arrival of the 2-simplices. Given a 2-simplex $(v_1,v_2,v_3)$ that appears at scale $r$ in filtration $\tilde{\mathcal{C}}(G)$, let us check that it appears at scale $kr$ in filtration $\tilde{\mathcal{C}}(G')$. 
    Let $x_1$, $x_2$ and $x_3$ (resp. $x_1'$, $x_2'$ and $x_3'$) represent the coordinates of $v_1$, $v_2$ and $v_3$ in $G$ (resp. $G'$) (recall that this makes sense as $G$ and $G'$ have the same set of nodes). 
    From the definition of $\tilde{\mathcal{C}}(G')$, the simplex $(v_1,v_2,v_3)$ appears at scale: 
    \begin{align*}
        & min\left\{d_{G'}(v_1,v_2)+d_{G'}(v_2,v_3),d_{G'}(v_1,v_3)+d_{G'}(v_2,v_3),d_{G'}(v_1,v_2)+d_{G'}(v_1,v_3)\right\} \\
        = & min\left\{k \ d_{G}(v_1,v_2)+k \ d_{G}(v_2,v_3),k \ d_{G}(v_1,v_3)+k \ d_{G}(v_2,v_3),k \ d_{G}(v_1,v_2)+k \ d_{G}(v_1,v_3)\right\} \\
        = & k \ min\left\{d_{G}(v_1,v_2)+d_{G}(v_2,v_3),d_{G}(v_1,v_3)+d_{G}(v_2,v_3),d_{G}(v_1,v_2)+d_{G}(v_1,v_3)\right\} \\
        = & k r  \qquad (\text{From assumption on the scale of appearance in $\tilde{\mathcal{C}}(G)$}) \,.
    \end{align*}
    
    This concludes the proof.
\end{proof}

\subsection{Equivariance of the topological spatial graph coarsening}\label{sec:optimal_coarsening_prop}

Given the two propositions above, we can now look at the whole topological spatial graph coarsening procedure and show that is it equivariant under the symmetric group action. 

\begin{proposition}
    Given a graph $G=(V,E,X)$ and a group element $(R,A,k)\in Sim(p)$, the topologically informed spatial coarsening method $F:\mathcal{G}\rightarrow\mathcal{G}$ where $F(G)=f_{\theta^*}(G)$ and  $\theta^*=argmin_{\theta\in\mathbb{R}_+}S_\theta(G)$ is equivariant under the symmetric group action:
    
    $$F(\psi(G,(R,A,k)))=\psi(F(G),(R,A,k))$$
    \label{prop:equivariance_final}
\end{proposition}

\begin{proof}
We want to prove that
\[F(\psi(G,(R,A,k)))=f_{\beta^*}(\psi(G,(R,A,k))\]
\[=f_{\beta^*}(V,E,kRX+A)=\psi(F(G),(R,A,k))=\psi(f_{\theta^*}(V,E,X),(R,A,k)) \,,\]
with $\beta^* = k \theta^*$. 
As we know from Proposition 1 that \[\psi(f_\theta(G),(R,A,k))=f_{\phi(\theta,k)}((V,E,kRX+A)),\] we need to prove that the optimal scaling parameters are related as follows: $$\beta^*:=argmin_{\beta\in \R_+}S_\beta(G')=\phi(\theta^*,k) = k \theta^* \,.$$

%Consider a graph $G=(V,E,X)$ and a graph $G'=(V,E,kRX+A)$. 
Let us recall the scoring function: $S_\theta(G)= \frac{|f_\theta^E(G)|}{|E|} + \lambda(G) \  d_B(PD(G),PD(f_\theta(G)))$. For $S_\beta(G')$, we showed in Proposition \ref{prop:coarsening} that the first additive term is the same for $\beta=k\theta$ as $f^E_\theta(G)=f^E_\beta(G')$.
    
From Proposition \ref{prop:pd} we have that $PD(\psi(G, (R,A,k)) = k \cdot PD(G)$. Similarly, we have 
    $$PD(f_\beta(G')) \overset{\text{Prop.~\ref{prop:coarsening}}}{=} PD( \psi(f_\theta(G), (R,A,k) ) \overset{\text{Prop.~\ref{prop:pd}}}{=} k \cdot PD(f_\theta(G)) \,.$$
 Now, we need to prove that $d_B(PD(G'),PD(f_{\beta}(G'))) = k \cdot d_B(PD(G),PD(f_\theta(G)))$. Recall that the bottleneck distance is defined as 
 $$\db(\mu, \nu) = \underset{\pi \in \Pi(\mu, \nu)}{\inf} \underset{x \in \pi(x)}{\sup} \|x - \pi(x)\|_\infty \,.$$
 %where in this case $\mu$ is the original set of points in the persistence diagram of $G$ and $\nu$ is the set of points corresponding to the filtration in the coarsened graph $f_\theta(G)$. 
 Remark that the scaling of the points by $k$ will not affect the couplings and as the $\norm{}_\infty$ is a positively homogeneous function we have that: 
    \begin{align*}
        &d_B(PD(G'),PD(f_{k\theta}(G'))) \\
        =& d_B(k \cdot PD(G), k \cdot PD(f_{\theta}(G))) \\
        =& k \cdot d_B( PD(G), PD(f_\theta(G))) \,.
    \end{align*}
    
    Moreover, the normalizing factors $\lambda(G)$ and $\lambda(G')$ satisfy
    \begin{align*}
         \lambda(G') &= \left[\max_{\beta \in \R_+} \ d_B(PD(G'),PD(f_\beta(G'))) \right]^{-1} \\
        &= \left[\max_{\theta \in \R_+} \ d_B(PD(G'),PD(f_{k \theta}(G'))) \right]^{-1}  \quad (\text{change of parameter $\beta=k \theta$})\\
        &= \left[\max_{\theta \in \R_+} \  k \cdot d_B(PD(G),PD(f_{\theta}(G))) \right]^{-1} \quad (\text{Previous result on Botleneck distance)} \\
        &= k^{-1} \left[\max_{\theta \in \R_+} \  d_B(PD(G),PD(f_{\theta}(G))) \right]^{-1} = k^{-1} \cdot \lambda(G) 
    \end{align*}
    
    Bringing together the two elements of the scoring function, we have that:
    \begin{align*}
        S_{k\theta}(G') =& \frac{|f_{k\theta}^E(G')|}{|E|} + \lambda(G') \ d_B(PD(G'),PD(f_{k\theta}(G'))) \\
        \overset{\text{Prop.~\ref{prop:coarsening}}}{=}  &\frac{|f_{\theta}^E(G)|}{|E|} + \lambda(G') \  d_B(PD(G'),PD(f_{k\theta}(G'))) \\
        = & \frac{|f_{\theta}^E(G)|}{|E|} + \frac{k}{k} \lambda(G) \  d_B(PD(G),PD(f_{\theta}(G))) = S_\theta(G)
    \end{align*}
This directly gives that 
$$\beta^*= \arg\min_{\beta \in \R_+} S_\beta(G') = k \cdot \arg\min_{\theta \in \R_+} S_{k \theta}(G')  = k \cdot \arg\min_{\theta \in \R_+} S_{\theta}(G) = k \theta^* \,,$$
which finishes the proof.
\end{proof}

\begin{remark}
The results in this section are based on the hypothesis that the nodes of the graph are embedded in spaces with a flat curvature ($CAT(0)$ curvature) - as the Euclidean Space - and that the edge length is the Euclidean distance between the coordinates. These hypotheses ensure that any transformation of the object is preserving the topology and the geometry of the graph. Different embedding of the nodes of the graph (for example an hyperbolic embedding of the nodes of the graph) change the effect of the group action, and consequently the coarsening procedure. However, these considerations are out of the scope of this paper.

\end{remark}

\section{Experiments and Real World Examples}
\label{sec:experiments}

We illustrate the proposed method on simulated and real world data. The simulated data consists of a spatial graph where the nodes are positioned on a ring with a random effect and the edges are sampled at random. We test the spatial graph coarsening procedure with average degree positioning. We then move to two real world dataset. In the first, we simplify the road network of the city of Marseille (France) using the degree positioning \citep{boeing2025modeling}. In the second, we test the effect of the topological spatial graph coarsening procedure of a classification task for a dataset of network fungi \citep{fricker2025fungi}.

\subsection{Practical Implementation}

When applying our topological spatial graph coarsening on a graph $G = (V,E,X)$, we need to optimize the score function $\theta \to S_\theta(G)$. Unfortunately, this optimization problem is untractable with iterative optimization algorithms (like gradient descent) as the score is a priori neither differentiable nor convex with respect to $\theta$. In practice, we evaluate the score on a grid $\Theta_m$ of size $m$ and select the minimizer on this grid. As the values of $\theta$ act as thresholds for the edge lengths, we do not use a linear grid and rather propose a grid that is adapted to the distribution of edge lengths. In practice, the grid is set as $\Theta_m = \{ q_\alpha(\mathcal{L}) , \alpha \in \{ 1/(m+1), \dots, m/(m+1) \} \}$, where $q_\alpha(\mathcal{L})$ is the empirical $\alpha$-quantile of the edge lengths $\mathcal{L} = \{ \ell_{u,v}, \ (u,v) \in E \}$. With this choice, when considering the next $\theta$ in the grid, we coarsen the graph by collapsing an additional proportion of edges of $1/(m+1)$.

In practice, the computation of the persistent diagram for the original graph and the ones of the reduced graphs for each $\theta$ in the grid $\Theta_m$ can become computationally heavy. To reduced this computational cost, we may limit the scales $r$ that we explore in the filtration and only consider a truncated version $(\widetilde{\mathcal{C}}_r(G))_{0 \leq r \leq r_{\max}}$. We may choose $r_{\max}$ as a fraction of the diameter in the graph, that is the maximal shortest-path distance over the pairs of nodes.

\subsection{Single Spatial Graph Coarsening}
We start by considering a single spatial graph generated in the following fashion. Consider $n=100$ points uniformly sampled on an annulus with outer radius $1$ and inner radius $0.7$. We include in the graph only a fraction $p=0.1$ of all the $n(n-1)/2$ possible edges, those with smallest length (given by Euclidean distance).  In Figure \ref{fig:annulus_reduction} (left), we can see original and the reduced graph. The position of the nodes in the coarsened graph follow the average positioning  procedure described in Section \ref{sec:graph_coarsening}. The optimal reduction is chosen by looking at the score values in Figure \ref{fig:annulus_reduction} (right). We plot the score $S_\theta(G)= |E(\theta)|/|E| + \lambda \cdot  \db(PD(G), PD(G(\theta)))$ defined as in Equation \eqref{eq:score} and its two components, the ratio of edges and the bottleneck distance scaled by $\lambda$ set as the inverse of the maximum of the bottleneck distance. The optimal threshold $\theta^\star$ is found at the minimum of the scoring function.
\begin{figure}[H]
\centering
\begin{minipage}{0.53\linewidth}
    \centering
    \includegraphics[width=\linewidth]{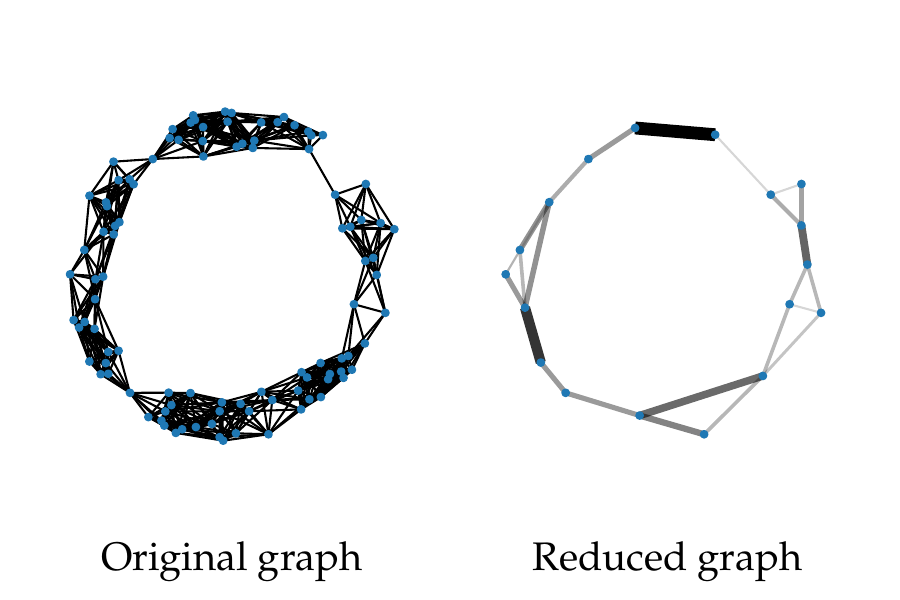}
    \caption{Synthetic annulus graph. \\ Left: original graph; right: reduced graph.}
    \label{fig:annulus_graphs}
\end{minipage}
\hspace{0.01\linewidth}
\begin{minipage}{0.44\linewidth}
    \centering
    \includegraphics[width=\linewidth]{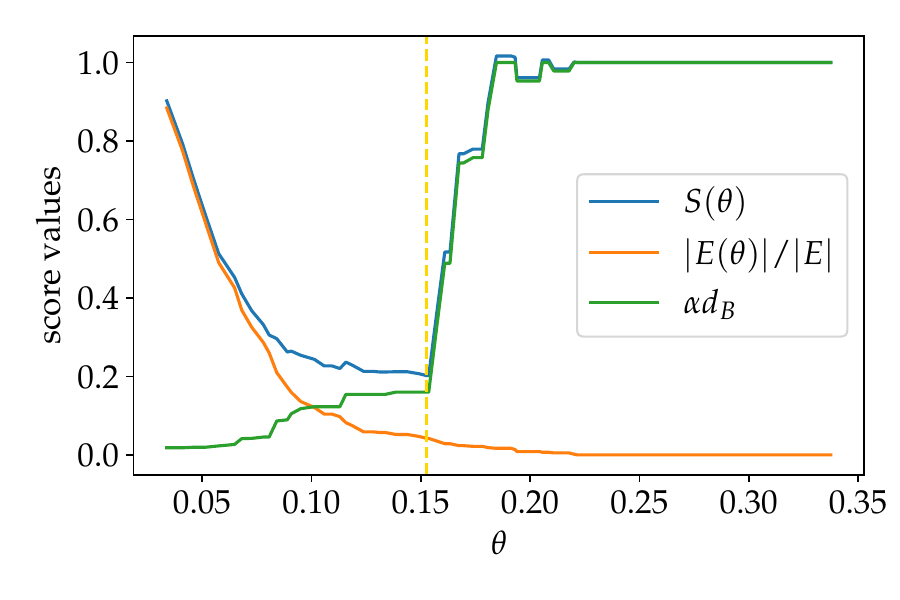}
    \caption{Annulus graph. Scores values with respect to the reduction threshold $\theta$.}
    \label{fig:annulus_reduction}
\end{minipage}
\end{figure}

As a real world example, we select the road network of the city of Marseille in south of France (the data are publicly available in the OSMnx package by \cite{boeing2025modeling}). The graph is made of all the portions of road of the following type: motorway, motorway link, primary road, secondary road,  secondary road link.
Road networks are a perfect example of spatial graphs, with natural constrains on the road topology and with high complexity in terms of nodes and edges. The original and the reduced graphs are presented in Figure \ref{fig:road_graphs}. The position of the nodes is computed using the degree positioning described in Section \ref{sec:graph_coarsening}. For each new hypernode, we choose the position of the original node with the highest connectivity - i.e. higher degree. In this road network example, this is equivalent to assign the position of the crossing between the highest number of streets.
\begin{figure}[H]
\centering
\begin{minipage}{0.53\linewidth}
    \centering
    \includegraphics[width=\linewidth]{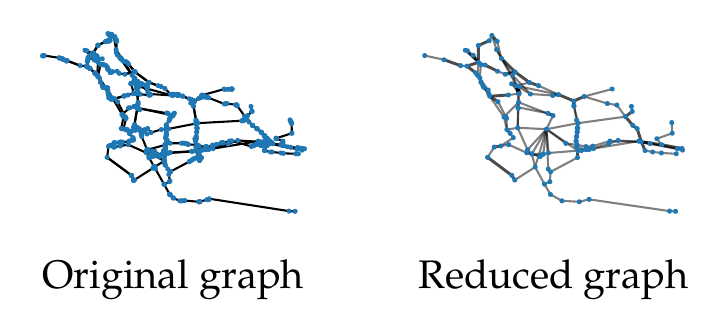}
    \caption{Road network of Marseille, France. \\ Left: original graph; right: reduced graph.}
    \label{fig:road_graphs}
\end{minipage}
\hspace{0.01\linewidth}
\begin{minipage}{0.44\linewidth}
    \centering
    \includegraphics[width=\linewidth]{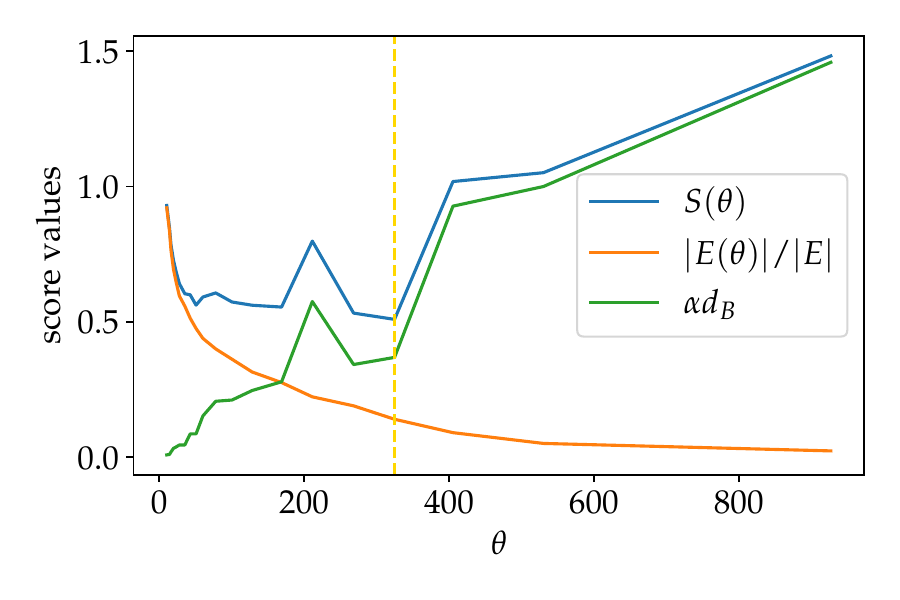}
    \caption{Score function for the road network coarsening with respect to the reduction threshold.}
    \label{fig:road_scores}
\end{minipage}
\end{figure}

\subsection{Application on a classification task}

We evaluate our topological and optimized coarsening procedure on a fungi network dataset. We consider a real-world dataset of fungal mycelia extracted from \cite{fricker2025fungi}. Each fungus is represented as a spatial graph where the edges represent mycelia \footnote{The mycelium is the root-like structure of a fungus.} branches and the nodes are placed where the mycelia branches separate or merge (see Figure~\ref{fig:fungus_images}). The dataset contains fungi of different species, grown in various conditions, and exposed to attacks of grazers of various types. Each type of grazer modifies the growth of the fungus, resulting in a change of the structural properties of the network \citep{boddy2010fungal} For example, the structural changes between two fungi grown with or without grazers are shown in Figure~\ref{fig:fungus_images}).
For more consistency through the dataset and for computational time considerations, we select a subset of the graphs with a number of nodes smaller than 1500, resulting in a data set of 128 graphs, spanning three fungus species (Phallus impudicus,  Phanerochaete velutina, Resinicium bicolour). The nodes of the graph are given with 2D spatial positions and the edges are weighted with biological-inspired resistance values that indicate how easy it is to transport nutrient along the given edge. This resistance value is proportional to the length of the edge and inversely proportional to the cross-section area.\\ 

\begin{figure}
    \centering
    \includegraphics[height=0.4\linewidth]{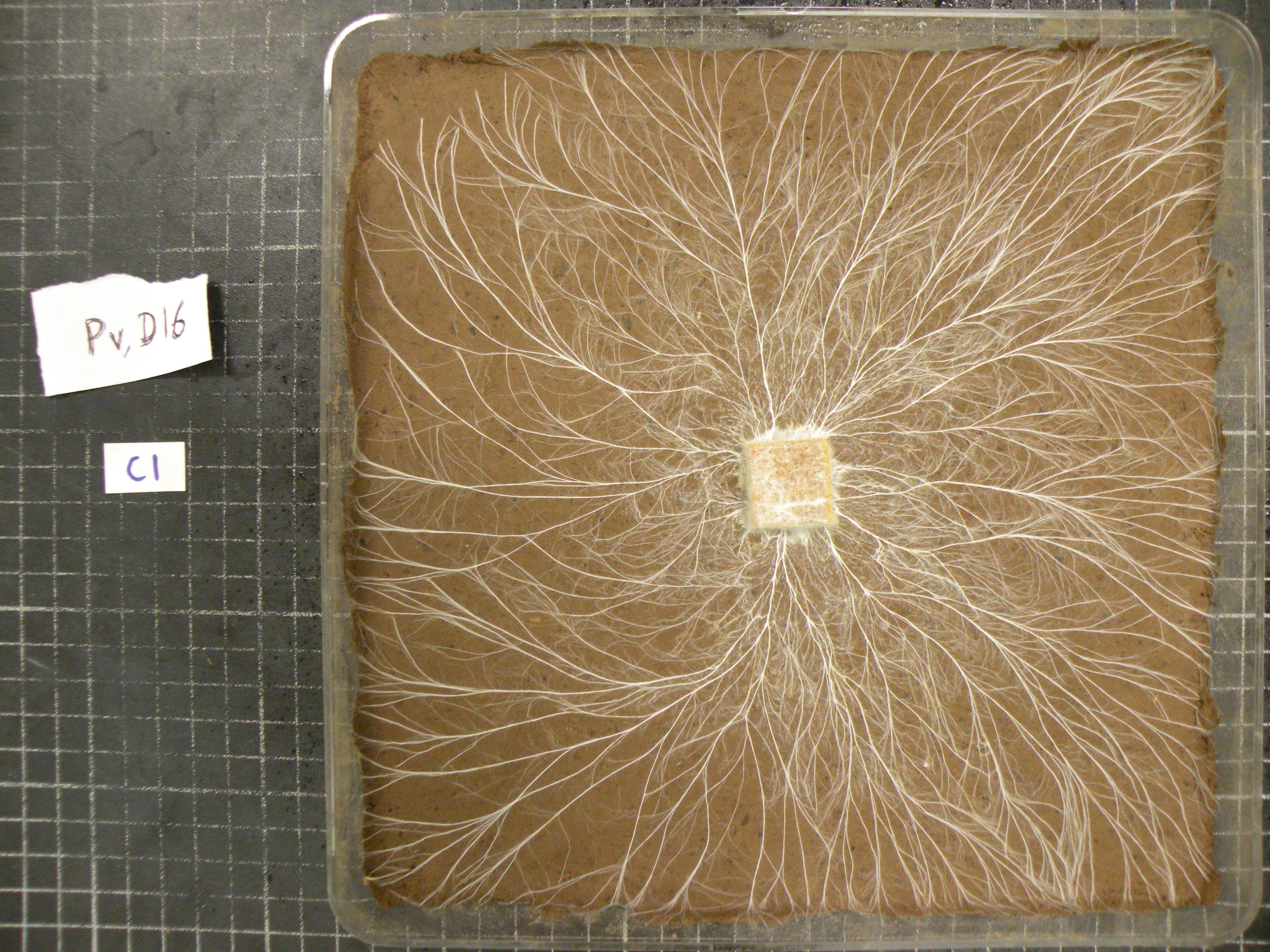}
    \includegraphics[height=0.4\linewidth]{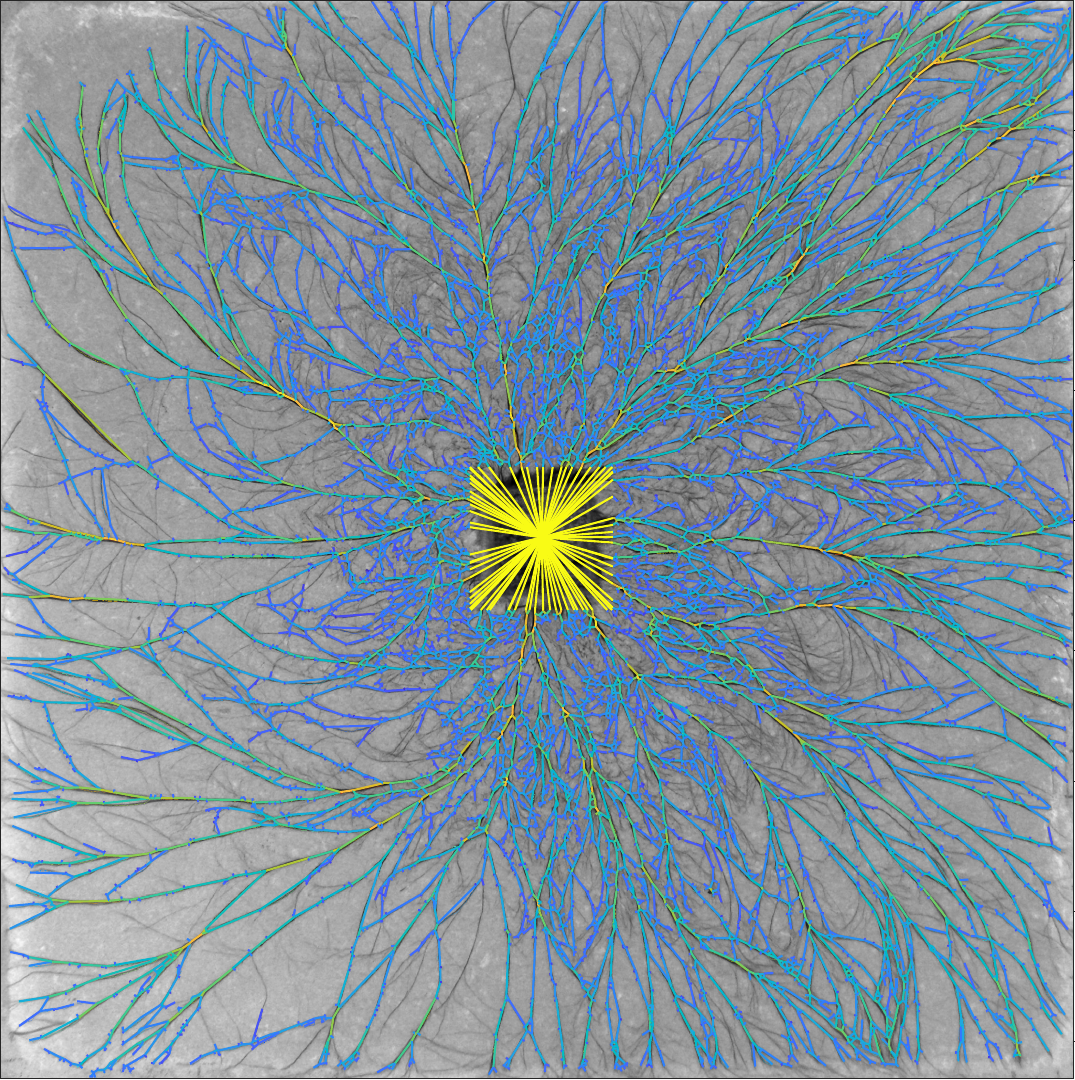}
    \includegraphics[height=0.4\linewidth]{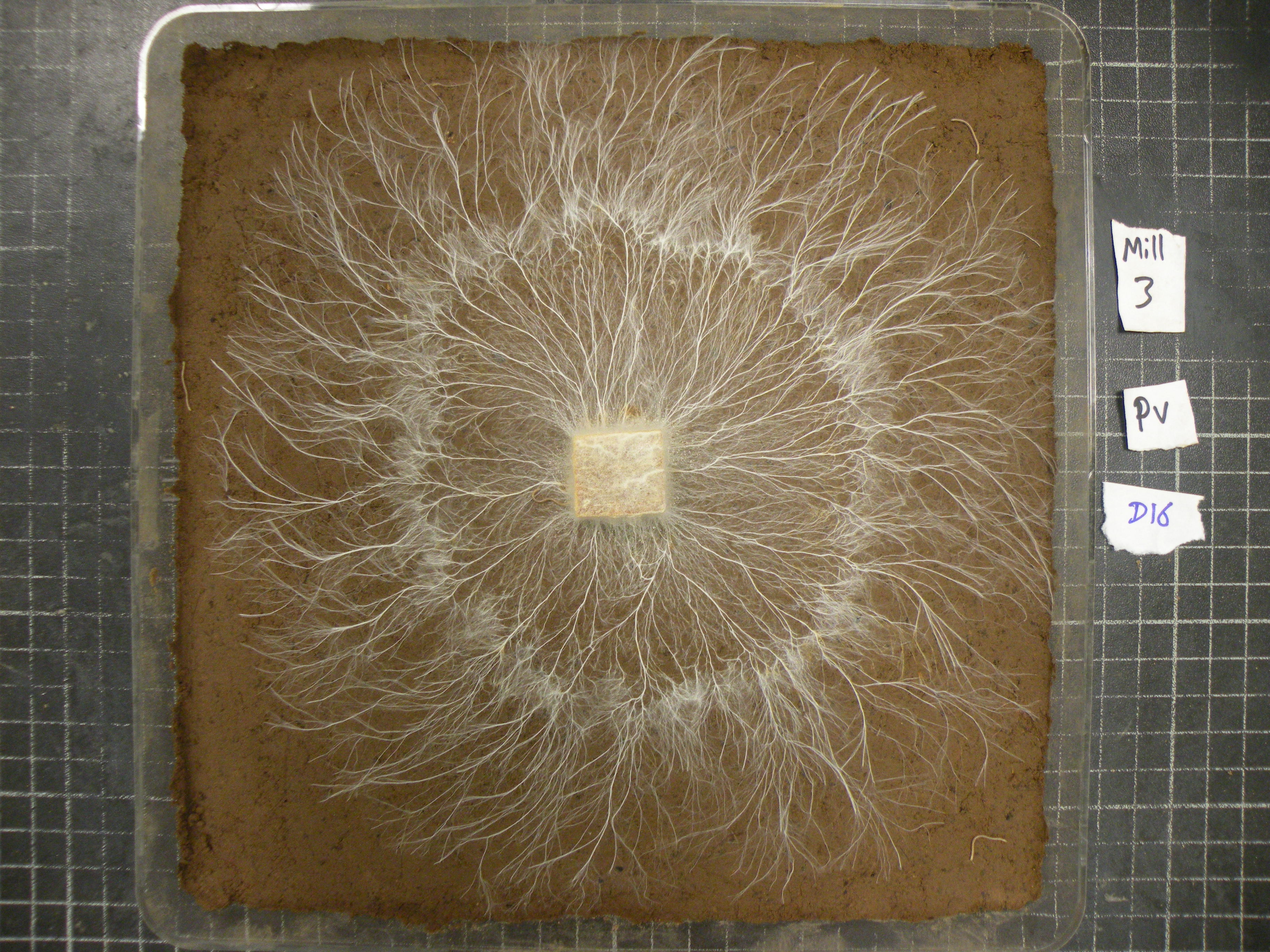}
    \includegraphics[height=0.4\linewidth]{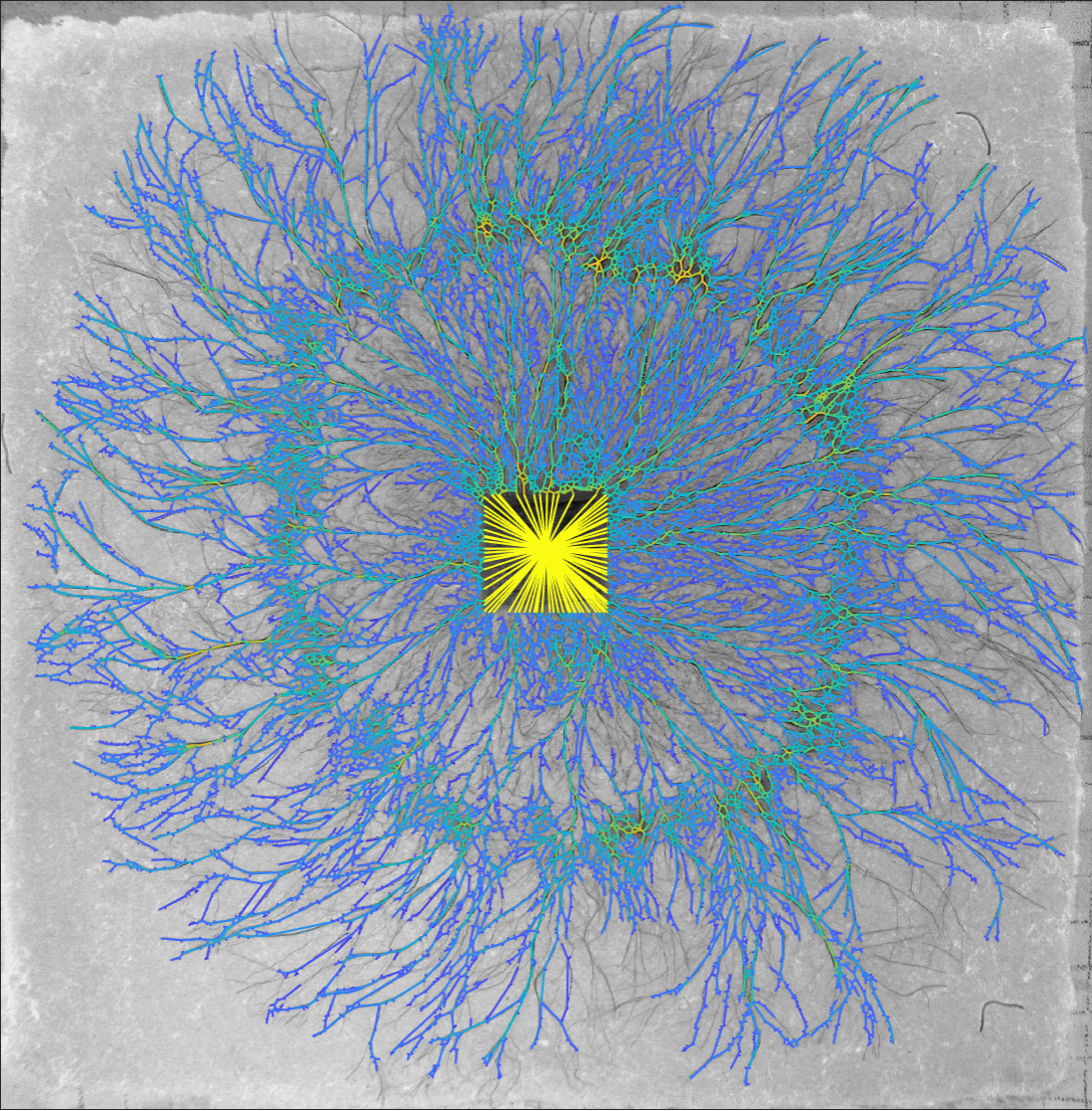}
    \caption{Examples of fungus mycelia and their corresponding graphs. Fungi are grown from the central woodblock. Top: Phanerochaete velutina grown without grazers. Bottom: Phanerochaete velutina grown with Millipedes grazers. The left images show pictures of the fungi in the experimental setting \citep{fricker2025fungi}. The right images show the graph structures that are extracted from the images (the wood block is treated as only one node) by \cite{fricker2025fungi}. Colors indicate weights computed based on length and thickness of the mycelium pieces.}
    \label{fig:fungus_images}
\end{figure}

We apply our topological spatial graph coarsening algorithm to each graph of the dataset. Note that here, as the graph is given with specific weight that are relevant for the biological modeling, the coarsening is perform by using the resistance weights instance of edge lengths. Observe that the resistance is still a spatially and geometrically inspired quantity, similar to the length. Thus it is relevant to use it in this biological context. 
For each graph $G = (V,E,X)$ in the dataset, we apply our topological spatial graph coarsening procedure but replacing the use of the edge lengths the specific resistance values. Therefore, the scoring function is optimized over a grid $\Theta_m = \{ q_\alpha(\mathcal{W}) , \alpha \in \{ 1/(m+1), \dots, m/(m+1) \} \}$ adapted to the values of the resistance weights $\mathcal{W} = \{ w_{u,v}, \ (u,v) \in E \}$. In the experiments, $m$ is set to 10.
Once we obtain the minimizer $\theta^*$ of $S$ over $\Theta_m$, we can save the corresponding quantile level $\alpha^*$ such that $q_{\alpha^*}(\mathcal{W}) = \theta^*$. This $\alpha^*$ corresponds to the fraction of edges that are collapsed when the coarsening is performed with threshold $\theta^*$. The histogram of these optimal quantile levels $\alpha^*$ over the dataset are given in Figure~\ref{fig:hist_quantile_levels}. We can observe that based on the topology of the graph, our coarsening algorithm chooses to reduce the graphs by different amounts. Nonetheless, it is able to significantly reduce the graph sizes, as indicated by the boxplots representing the graph size distribution on both the original graph dataset and the dataset of reduced graphs in Figure~\ref{fig:boxplot_sizes}. Examples of the obtained coarsened graphs are shown in Figures \ref{fig:Pi_ctrl2_d4_5_graph}, \ref{fig:Pi_coll_d4_5_graph} and \ref{fig:Pi_mill_d8_2_graph}, with the corresponding score functions in Figures \ref{fig:Pi_ctrl2_d4_5_scores}, \ref{fig:Pi_coll_d4_5_scores}, and \ref{fig:Pi_mill_d8_2_scores}.

\begin{figure}[H]
    \begin{minipage}{0.5\linewidth}
        \centering
        \includegraphics[width=\linewidth]{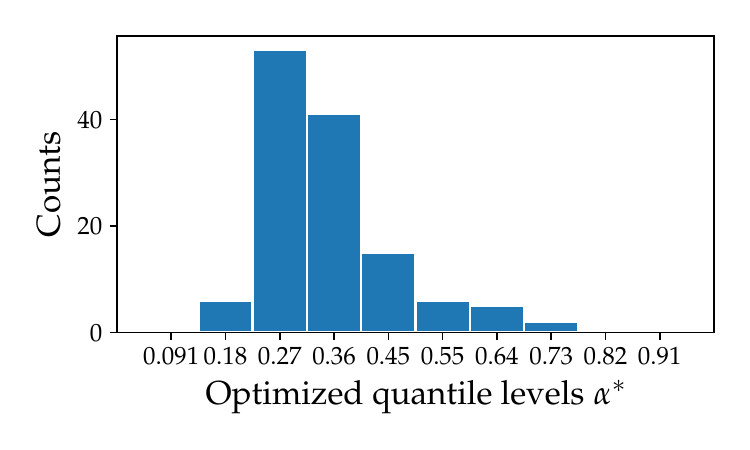}
        \caption{Histogram of the optimal quantiles levels $\alpha^*$ over the selected dataset. The bins of the histogram matches the grid $\Theta_m$.}
        \label{fig:hist_quantile_levels}
    \end{minipage}
    \begin{minipage}{0.45\linewidth}
        \centering
        \includegraphics[width=\linewidth]{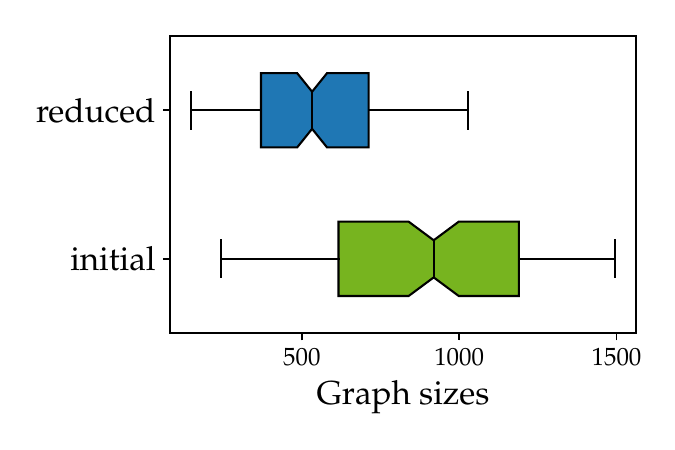}
        \caption{Boxplot illustrating the distribution of the graph sizes, both for the original selected dataset and the reduced one. }
        \label{fig:boxplot_sizes}
    \end{minipage}
\end{figure}

\begin{figure}[H]
\centering
\begin{minipage}{0.53\linewidth}
    \centering
    \includegraphics[width=\linewidth]{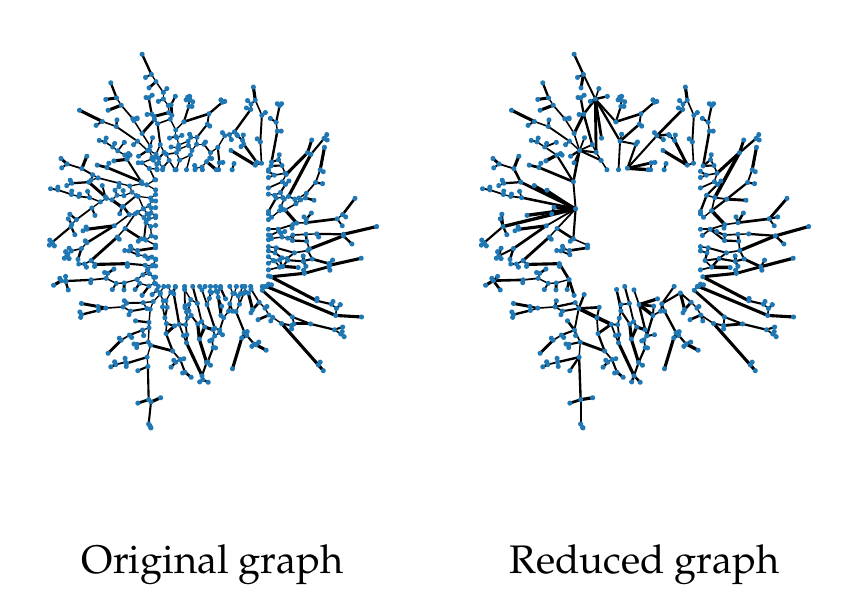}
    \caption{Illustration of the topological spatial reduction on a ``Phallus impudicus`` fungus without grazing.}
    \label{fig:Pi_ctrl2_d4_5_graph}
\end{minipage}
\hspace{0.01\linewidth}
\begin{minipage}{0.44\linewidth}
    \centering
    \includegraphics[width=\linewidth]{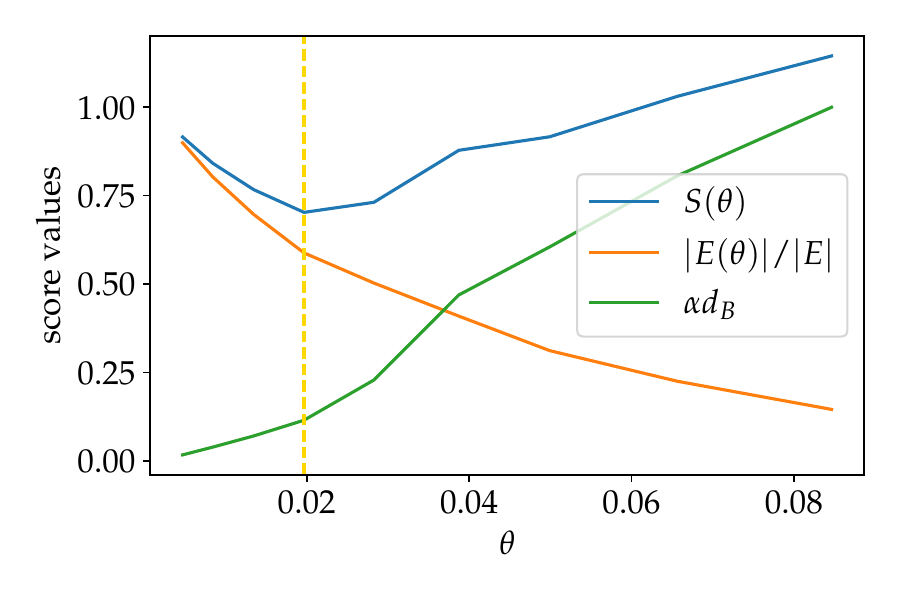}
    \caption{The score $S(\theta)$ (in blue), the graph complexity (in green) and the topological distance (in green) w.r.t. $\theta$, associated to the graph reduction in Figure~\ref{fig:Pi_ctrl2_d4_5_graph}.}
    \label{fig:Pi_ctrl2_d4_5_scores}
\end{minipage}
\end{figure}

\begin{figure}[H]
\centering
\begin{minipage}{0.53\linewidth}
    \centering
    \includegraphics[width=\linewidth]{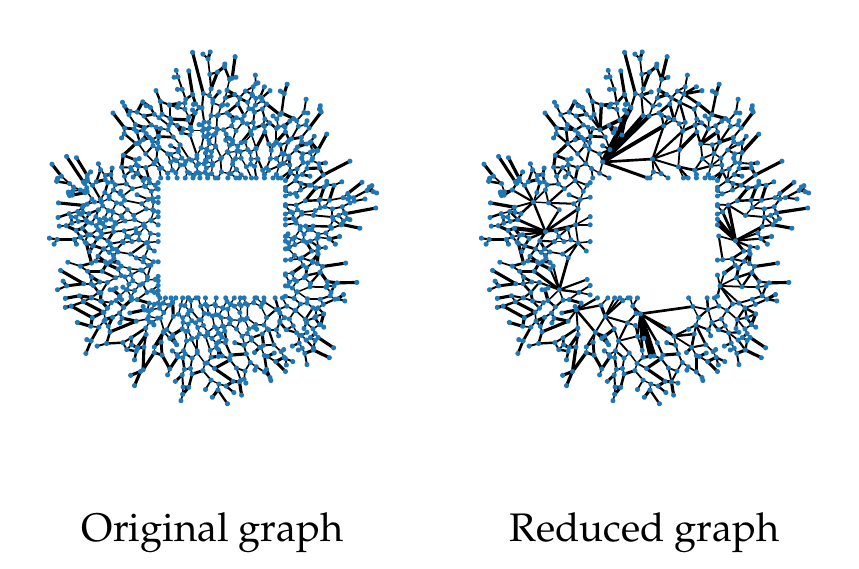}
    \caption{Illustration of the topological spatial reduction on a ``Phallus impudicus`` fungus with grazing from ``collembola (small grazers).}
    \label{fig:Pi_coll_d4_5_graph}
\end{minipage}
\hspace{0.01\linewidth}
\begin{minipage}{0.44\linewidth}
    \centering
    \includegraphics[width=\linewidth]{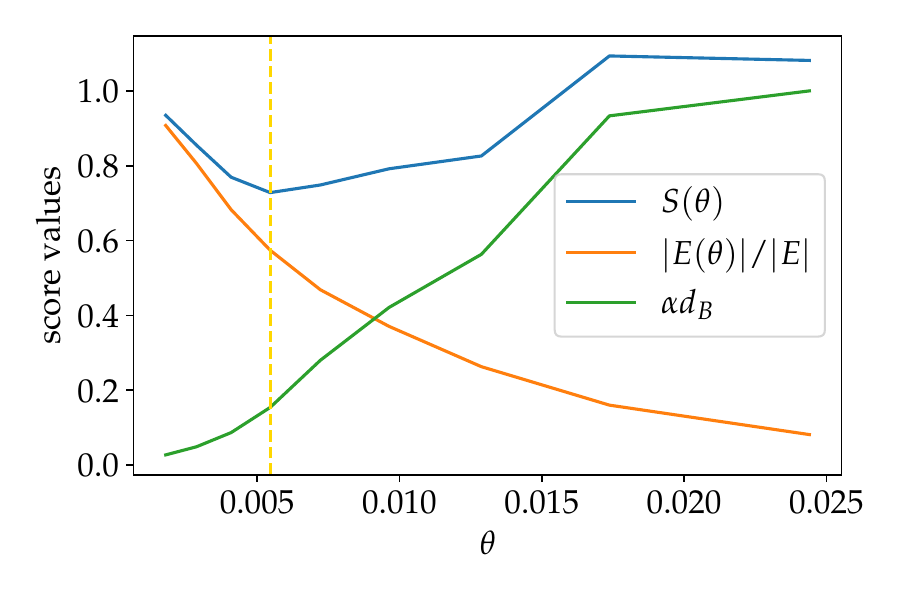}
    \caption{Same as above for the graph reduction in Figure~\ref{fig:Pi_coll_d4_5_graph}.}
    \label{fig:Pi_coll_d4_5_scores}
\end{minipage}
\end{figure}

\begin{figure}[H]
\centering
\begin{minipage}{0.53\linewidth}
    \centering
    \includegraphics[width=\linewidth]{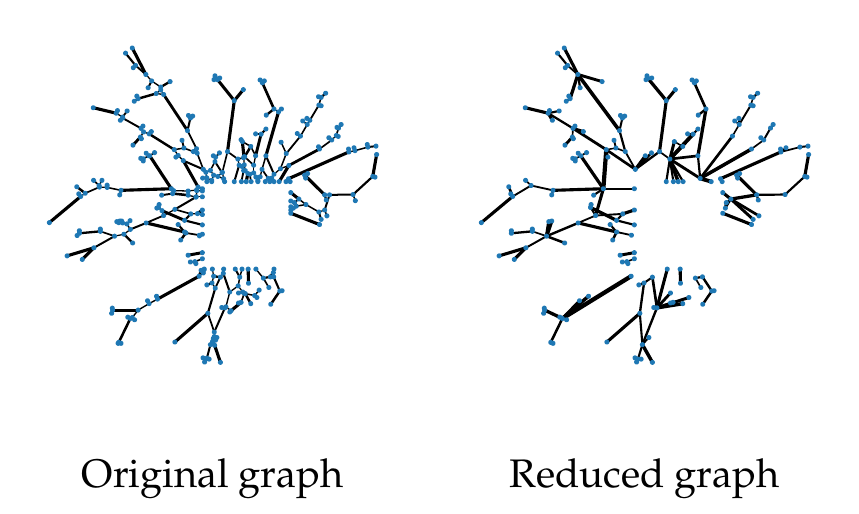}
    \caption{Illustration of the topological spatial reduction on a ``Phallus impudicus`` fungus with grazing from millipedes (large grazers).}
    \label{fig:Pi_mill_d8_2_graph}
\end{minipage}
\hspace{0.01\linewidth}
\begin{minipage}{0.44\linewidth}
    \centering
    \includegraphics[width=\linewidth]{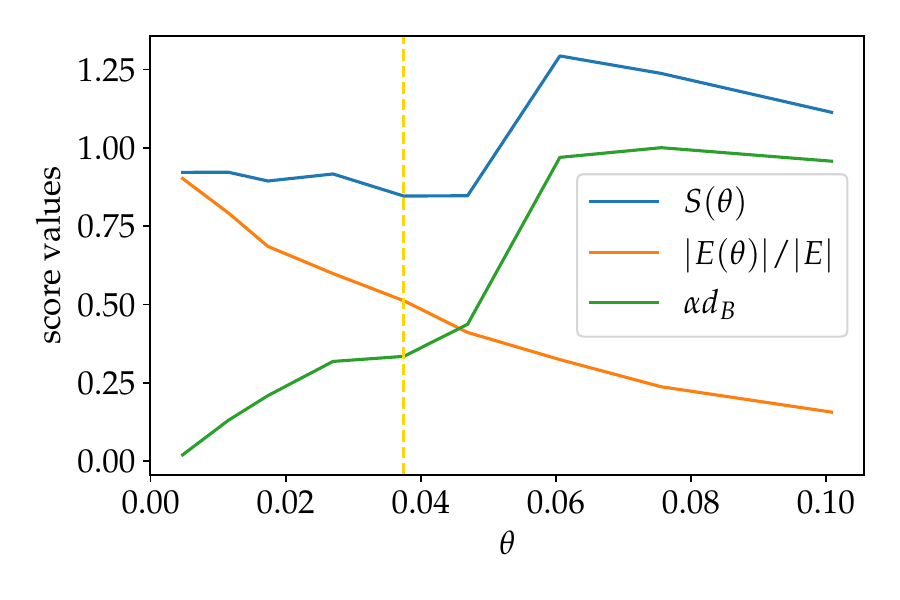}
    \caption{Same as above for the graph reduction in Figure~\ref{fig:Pi_mill_d8_2_graph}.}
    \label{fig:Pi_mill_d8_2_scores}
\end{minipage}
\end{figure}

To evaluate how well the information has been preserved in the reduced graph, we conduct the following experiment. We perform a classification task to predict the type of grazers that has attacked the fungus on both the original and the coarsened dataset. We then compare the accuracy of both classifiers, to estimate how much of the principal original information the coarsening method was able to retain. The experiment is based on the fact that the glazer type effect the grow of the fungal mycelium network \citep{fricker2025fungi}. For this experiment, we consider 3 types of grazing scenarios (3 classes): no grazer (control group), small grazers (< 10mm)  including collembola, mites and nematodes, and large grazers (> 10mm) including millipedes, woodlices, and enchytraeid worms. 
To perform the classification, for each spatial graph we extract simple features from their persistent diagrams: number of connected components, mean, maximal and total persistence (where the persistence of a point in the diagram is the difference between its death and birth value) for the 1D persistence diagram, mean birth and death values for the 1D persistence diagram, $L_2$ norm of the concatenated 5 first persistence landscapes \citep{bubenik2015statistical}, and number of degree 1 nodes in the graph. For each dataset (original and reduced), we extract these features. A random forest classifier is trained for each dataset, with 200 trees. Performances are evaluated by the accuracy on a 10-fold cross validation. The accuracies of the classifiers are reported in Figure~\ref{fig:acc_errbar}. The error bars represent the 95\% confidence interval obtained by bootstrapping $10^4$ times the predictions on each sample given by the cross-validation procedure.
Here, one should note that the dataset considered is not very large, therefore performances of the random forest is not expected to be very high. Nonetheless, we see that our learning procedure is able to capture relevant information as the obtained accuracies are significantly above 1/3, which would correspond to a totally uninformative random classification. Secondly, and more importantly, we observe only a slight accuracy drop between the original and reduced graphs, that does not seem to be statistically significant, as indicated by the overlapping confidence intervals. 
This clearly indicates that our coarsening algorithm has been able to significantly reduce the graph sizes by roughly a factor 2 (see Figure~\ref{fig:boxplot_sizes}) while still preserving the key information contained in these spatial graphs such that learning algorithms can still perform as good as when applied on the original data. We emphasize here that we did not tune the amount of reduction applied by our coarsening method to preserve classification performance as the whole point of our proposed approach is to have a self-tuning coarsening algorithm. The reduction level is set by preserving the topology of the graph, which is often carrying meaningful information, as demonstrated in the present experiment.

\begin{figure}
    \centering
    \includegraphics[width=0.55\linewidth]{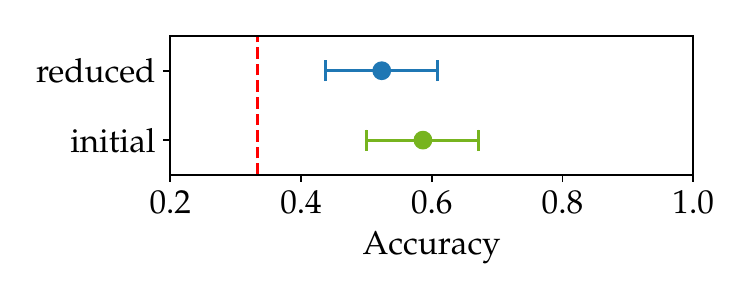}
    \caption{Accuracy of the random forest classifiers, on both the original and reduced datasets. The 95\% confidence intervals are computed by bootstrapping $10^4$ times the predictions of the 10-fold cross-validation. Vertical dashed line represents the theoretical accuracy of a random classifier (3 classes).} 
    \label{fig:acc_errbar}
\end{figure}

\section{Conclusion and further developments}
We propose a new method for coarsening graphs with spatial coordinated associated to the nodes. The methods aims at preserving the topological properties of the graph, captured by adapting standard persistence diagram. We prove the novel coarsening method is equivariant with respect to the action of the symmetric group. We tested the method on few simulated data and two real world datasets: a road network and a set of fungal mycelia. To grasp the trade off between retaining salient information and simplifying the graphs we perform a classification task on the fungi dataset. While the coarsening procedure significantly reduced the graph sizes, we observe only a small decrease in the accuracy, indicating that our coarsening approach has been able to preserve the key topological information used for the classification task.  

There are several further development available. Firstly, the optimization procedure on the $\theta$ can be refined by for example using a smart discretization procedure (e.g. domain adaptation task). 
It might be worth exploring more some parts of the edge lengths distribution, while others (\eg, corresponding to drastic graph reductions) should be investigated less thoroughly.  
Along with spatial graphs, shape graphs have also been defined as similar data object accounting for shape along the edges (for example the shape of a vein) \citep{bal2024statistical}. An adaptation to such data could be defined by accounting the edge shape into both the PD estimation and the coarsening procedure. 

A final yet relevant further development is to define an inverse procedure for which we identify the edges that should be collapsed in order to preserve certain topological features. Such procedure relies on the open question of how to inverse the diagram and identify graph elements that are associated to some points in the persistent diagram. 

\section*{Acknowledgments}

We would like to thank Mark Fricker for the valuable discussions on the fungus data and for providing us with the dataset. We also want to acknowledge the financial support of the CNRS-Imperial "Abraham de Moivre" International Research Laboratory on this project.

\section*{Statements and Declarations}
\noindent\textbf{Conflict of interest} On behalf of all authors, the corresponding author states that there is no conflict of interest.\\
\textbf{Funding} Partial financial support was received from the CNRS-Imperial "Abraham de Moivre" International Research Laboratory Travel fund.

\bibliographystyle{./bst/sn-mathphys-ay}
\bibliography{bib.bib}
\end{document}